\documentclass[11pt]{article}

\usepackage[final]{acl}

\usepackage{times}
\usepackage{latexsym}

\usepackage[T1]{fontenc}
\usepackage[utf8]{inputenc}

\usepackage{microtype}

\usepackage{inconsolata}

\usepackage{graphicx}

\usepackage{amsmath}
\usepackage{amssymb}
\usepackage{mathtools}
\usepackage{amsthm}
\usepackage[noabbrev,capitalize,nameinlink]{cleveref}
\theoremstyle{plain}
\newtheorem{theorem}{Theorem}[section]

\newtheorem{proposition}[theorem]{Proposition}

\theoremstyle{definition}

\theoremstyle{remark}

\usepackage{algorithm}
\usepackage{algorithmic}
\usepackage[algo2e]{algorithm2e}
\usepackage[utf8]{inputenc}
\usepackage[cjk]{kotex}
\usepackage{wrapfig}
\usepackage{subfigure}
\usepackage{multirow}
\usepackage{pbox}
\usepackage{graphicx}
\usepackage{bm}
\usepackage{siunitx}
\usepackage{booktabs}
\usepackage{lipsum}
\usepackage{caption}
\usepackage{cancel}
\usepackage[toc,page,header]{appendix}
\usepackage{minitoc}

\usepackage[textsize=tiny]{todonotes}
\usepackage{marginnote}
 
\usepackage{listings}
\definecolor{codegreen}{rgb}{0,0.6,0}
\definecolor{codegray}{rgb}{0.5,0.5,0.5}
\definecolor{codepurple}{rgb}{0.58,0,0.82}
\lstdefinestyle{mystyle}{
    commentstyle=\color{codegreen},
    keywordstyle=\color{magenta},
    numberstyle=\tiny\color{codegray},
    stringstyle=\color{codepurple},
    basicstyle=\ttfamily\footnotesize,
    breakatwhitespace=false,         
    breaklines=true,                 
    captionpos=b,                    
    keepspaces=true,                 
    numbers=left,                    
    numbersep=5pt,
    showspaces=false,                
    showstringspaces=false,
    showtabs=false,                  
    tabsize=2
}
\definecolor{diffstart}{named}{gray}
\definecolor{diffincl}{named}{green}
\definecolor{diffrem}{named}{orange }
\lstdefinelanguage{diff}{
basicstyle=\ttfamily\small,
morecomment=[f][\color{diffstart}]{@@},
morecomment=[f][\color{diffincl}]{+\ },
morecomment=[f][\color{diffrem}]{-\ },
}
\definecolor{bluegray}{rgb}{0.4, 0.6, 0.8}
\definecolor{electriclime}{rgb}{0.8, 1.0, 0.0}
\definecolor{malachite}{rgb}{0.04, 0.85, 0.32}
\definecolor{darkred}{rgb}{0.55, 0.0, 0.0}
\definecolor{darkblue}{rgb}{0.0, 0.0, 0.55}
\definecolor{darkgreen}{rgb}{0.0, 0.2, 0.13}
\definecolor{darkorchid}{rgb}{0.6, 0.2, 0.8}

\crefname{proposition}{proposition}{propositions}
\Crefname{proposition}{Proposition}{Propositions}
\Crefname{figure}{Fig.}{Figs.}% {<type>}{<singular>}{<plural>}
\definecolor{commentcolor}{RGB}{110,154,155}   % define comment color
\usepackage[table]{xcolor}
\title{Can Spectral-Clipping Enable Better Learning\\While Forgetting Less for Low-Rank Adaptation?}

\author{
 \textbf{Hyowon Wi\textsuperscript{1}} and
 \textbf{Noseong Park\textsuperscript{1}}
\\
 \textsuperscript{1}Korea Advanced Institute of Science and Technology (KAIST)
\\
 \small{
   \texttt{\{hyowon.wi,noseong\}@kaist.ac.kr}
 }
}

\begin{document}
\maketitle
\begin{abstract}
In recent years, low-rank adaptation (LoRA) has emerged as a significant paradigm that freezes pre-trained weights and introduces small, learnable adapters instead of fine-tuning the full set of parameters. In this work, we uncover several key insights regarding the \textit{singular} components of network parameters based on Singular Value Decomposition (SVD).
Firstly, the \textit{principal} singular components with large singular values in pre-trained network parameters can be effectively reused during fine-tuning, whereas the \textit{minor} components with smaller singular values are more task-specific and require substantial adaptation. Secondly, we first establish the theoretical connection that the uncontrolled growth of singular values in LoRA adapters leads to the forgetting of pre-trained knowledge — a well-known issue referred to as \textit{catastrophic forgetting}.
Building on these observations, we propose \textbf{SCLoRA}, which injects parameterized singular components with spectral clipping into the pre-trained model in a way that is aware of the spectral distribution of the pre-trained model. \textbf{SCLoRA} effectively adapts to new tasks by focusing updates on components that require adaptation, while simultaneously alleviating catastrophic forgetting. We conduct extensive experiments and demonstrate that \textbf{SCLoRA} not only improves downstream performance but also effectively retains pre-trained knowledge.

\end{abstract}

\section{Introduction}\label{sec:introduction}

Pre-trained language models (PLMs) have achieved remarkable performance in various natural language processing tasks~\citep{devlin2019bert,liu2019roberta,lan2019albert,he2020deberta,touvron2023llama,achiam2023gpt4,anil2023palm}. The common way to adapt pre-trained language models to downstream tasks is \textit{fine-tuning}. However, fine-tuning all parameters and storing copies of the large model for each downstream task results in significant cost and memory consumption. To address this issue, recent studies suggest parameter-efficient fine-tuning (PEFT) methods~\citep{hu2021lora,zhang2023adalora,lialin2023relora,liu2024dora,jiang2024mora,meng2024pissa,wang2025milora}, fine-tuning with only a small number of parameters. Low-Rank Adaptation (LoRA)~\citep{hu2021lora}, a representative PEFT approaches that keeps the pre-trained weights frozen and updates only a small number of parameters, has shown promising performance while being both storage- and compute-efficient.
% LoRA is designed based on the assumption that pre-trained language models are inherently low-dimensional and can learn efficiently even with random projections into smaller subspaces. The low-rank matrices serve as adapters, amplifying features that were learned but not emphasized during pre-training. 
% over other methods such as prompt tuning~\citep{lester2021prompttuning} or prefix tuning~\citep{li2021prefixtuning}

\begin{figure*}[t]
    \centering
    \subfigure[Fisher overlap]{\includegraphics[width=0.33\textwidth]{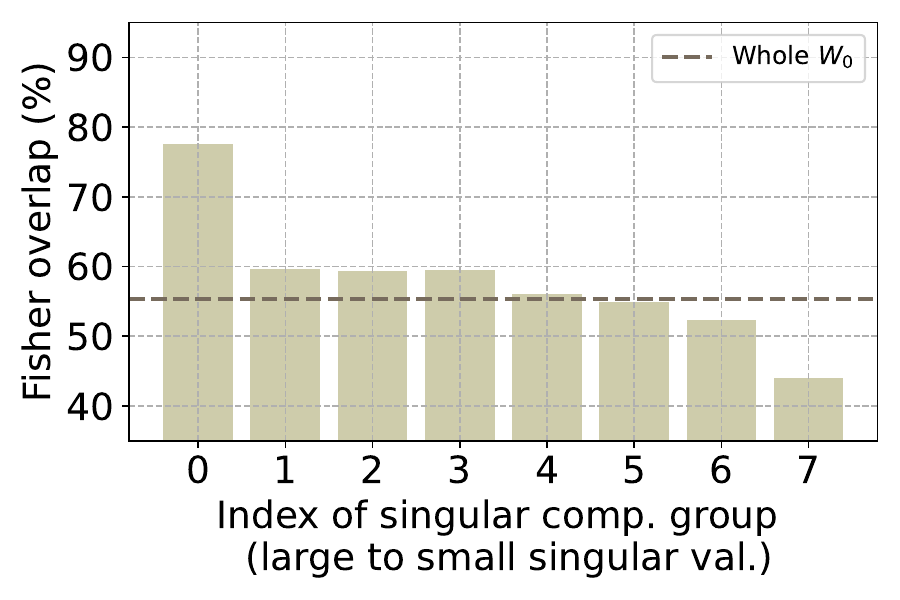}}\hfill
    \subfigure[Trade-off between spectral norm and pre-trained accuracy]{\includegraphics[width=0.345\textwidth]{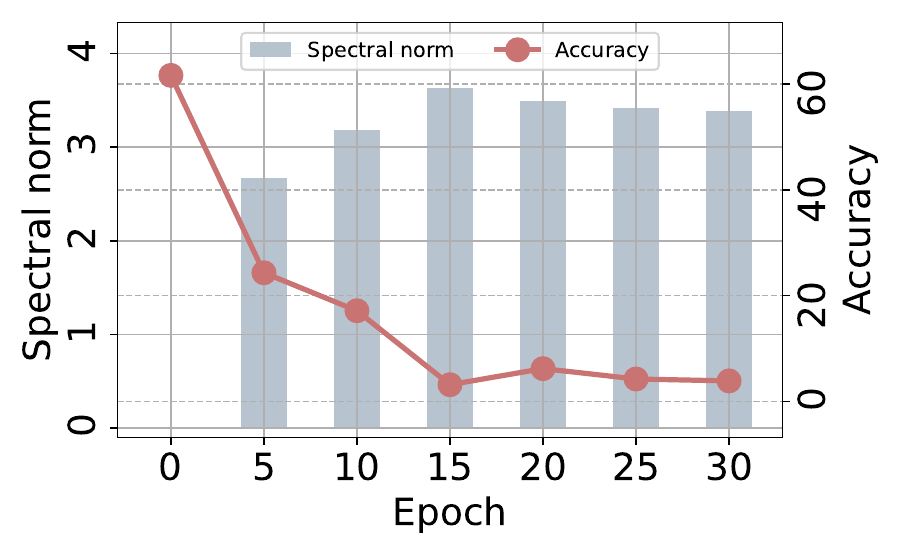}}\hfill
    \subfigure[Pre-trained accuracy trajectory]{\includegraphics[width=0.32\textwidth]{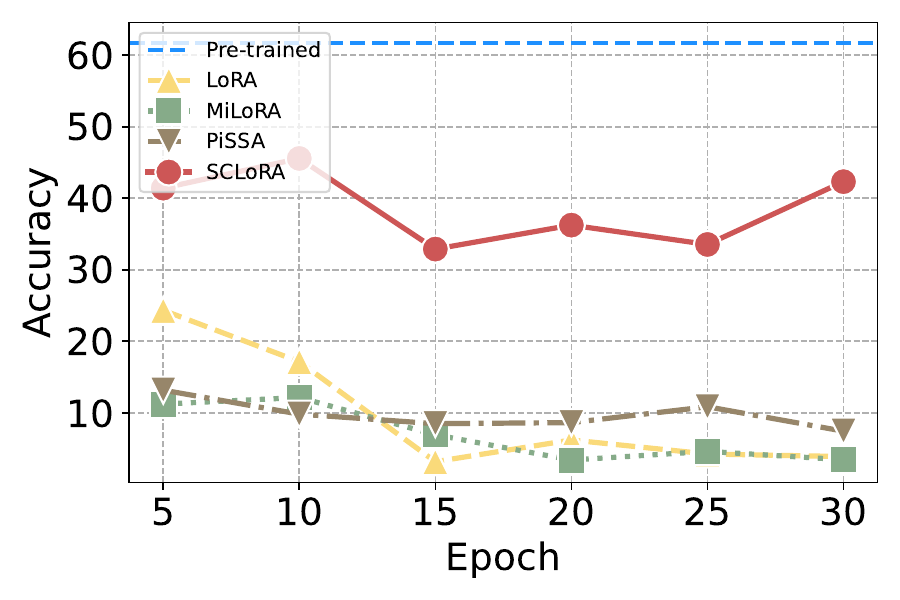}}
    % \jeongwhan{JC: Legend fontsize 작음} 
    \caption{(a) The Fisher overlap~\citep{kirkpatrick2017overcoming} between the pre-trained task (BookCorpus) and the fine-tuning task (MRPC), evaluated over partial reconstructions of the pre-trained network parameters obtained by grouping singular components sorted from large to small according to their singular values. Additional visualizations for other datasets are in \Cref{app:fisher_overlap_vis}. (b) The trade-off between the spectral norm of the adapter and the accuracy on the pre-trained task (BookCorpus) during fine-tuning of LoRA on the STS-B dataset from the GLUE benchmark for RoBERTa\textsubscript{base}. (c) Pre-trained accuracy trajectory across fine-tuning epochs with various baselines.}
    \label{fig:intro}
\end{figure*}

In recent years, many studies have investigated the properties of \textit{singular components} with Singular Value Decomposition (SVD) in LoRA~\citep{zhang2023adalora,meng2024pissa,wang2025milora,balazy2024loraxs,yang2024corda,wang2024loraga}. A singular component refers to a rank-1 matrix formed by the product of a pair of left and right singular vectors and their corresponding singular value. Specifically, a \textit{principal} singular component refers to one associated with a larger singular value, representing the global structure of the matrix~\citep{abdi2010principal, meng2024pissa}.
Conversely, a \textit{minor} singular component corresponds to a smaller singular value and is often considered as noise~\citep{wang2025milora}.
In deep learning, however, because learned weight matrices are typically full rank~\citep{hu2021lora,yu2023compressing,garg2025revealing}, the minor singular components are not merely noise; rather, they also encode detailed information within the matrix.

\paragraph{Motivations.} 
Recent studies have explored SVD-based variants of LoRA by selecting and modifying specific singular components~\citep{meng2024pissa,wang2025milora}. These approaches assume that principal components contain important information and minor components are merely noise. However, a principled understanding of how different singular components contribute to task adaptation and knowledge preservation is still missing.
To address this, we conduct a systematic analysis of LoRA through the lens of singular components. 
The first evidence is that the principal singular components of the pre-trained network parameters can be reused for the fine-tuning task to a great extent; the minor singular components become more task-specific and thus require a significant adaptation. 
To quantify spectrum-wise task alignment between pre-training and fine-tuning, we compute the \textit{Fisher overlap}~\citep{kirkpatrick2017overcoming, yao2022path, qian2024learn} on spectrally decomposed pre-trained parameters.
Specifically, we apply SVD, sort singular values in descending order, and group the corresponding singular components. Fisher overlap is then computed on partial reconstructions from each group, enabling alignment to be analyzed across spectral scales.
A higher Fisher overlap indicates stronger alignment, suggesting that the corresponding components are more transferable.
The detailed formulation of the Fisher overlap is provided in~\Cref{app:fisher_overlap_vis}.
\Cref{fig:intro} (a) shows that overlap gradually decreases as the singular components correspond to smaller singular values, suggesting that the principal singular components are already aligned, while minor singular components need more task-specific adaptation.

Beyond identifying where adaptation should occur, we further investigate how adaptation affects pretrained knowledge. 
Crucially, we establish the first theoretical connection showing that the growth of adapter singular values leads to the reduction of pretrained priors (see \Cref{thm:map}) during the fine-tuning optimization process. 
This reduction can result in a phenomenon called \textit{catastrophic forgetting}, where the model rapidly forgets pre-trained knowledge during fine-tuning. Such forgetting undermines the scalability and reliability of pre-trained models, making it essential to address this issue~\citep{wang2025milora,yang2024corda,ren2024analyzing,yang2024moral,dou2024loramoe}.
It is known that, however, during typical stochastic optimization, the spectral norm of weight matrices tends to grow rapidly (See \Cref{sec:theoretical_analysis} for details). Consequently, this norm growth leads to catastrophic forgetting in common fine-tuning settings. Empirically, \Cref{fig:intro}(b) shows that LoRA exhibits a significant increase in the singular values of its adapter during fine-tuning. As in~\Cref{fig:intro} (c), this increase is associated with performance degradation on the pre-trained task, indicating that LoRA is vulnerable to catastrophic forgetting. This phenomenon exposes a fundamental limitation of existing PEFT methods: unconstrained spectral growth results in significant forgetting.

\paragraph{Main idea.}
Motivated by these insights, we propose a \textbf{Lo}w-\textbf{R}ank \textbf{A}daptation with \textbf{S}pectral \textbf{C}lipping, called \textbf{SCLoRA}, a new PEFT framework that directs adaptation through the injection of spectrally clipped singular components, preventing the uncontrolled spectral growth observed in conventional updates.
\textbf{SCLoRA} constructs adapters as singular components through a parameterized SVD and constrains their singular values using a distribution-aware bound derived from the pretrained spectral distribution. This constraint prevents excessive spectral growth and guides learning toward the relatively minor regions of the spectrum. As shown in our analysis, this aligns updates with the spectral components that genuinely require adaptation, enabling effective downstream learning. Moreover, since our theoretical results indicate that increasing singular values reduces the pretrained prior, bounding them according to the pretrained spectrum helps preserve the prior knowledge encoded during pre-training, thereby mitigating catastrophic forgetting.
We conduct extensive experiments to evaluate the effectiveness of \textbf{SCLoRA}, demonstrating that it consistently outperforms LoRA and its variants across various tasks. Additionally, we assess catastrophic forgetting across multiple baseline models, showing that \textbf{SCLoRA} significantly mitigates the forgetting of pre-trained knowledge. Our key contributions can be summarized as follows:
\begin{itemize} 
    \item To our knowledge, we are the first to show that the minor singular components of network parameters require substantial adaptation, and we theoretically demonstrate that the growth of adapter singular values results in catastrophic forgetting.
    \item In \Cref{subsec:method}, we propose \textbf{SCLoRA}, which injects the singular components with spectral clipping using parameterized SVD, ensuring that the pre-trained model efficiently adapts to new tasks and mitigates the catastrophic forgetting problem.
    \item In \Cref{sec:experiments} and \Cref{sec:discussion}, extensive experiments demonstrate that \textbf{SCLoRA} efficiently adapts to the new task with performance gain and significant reduction in catastrophic forgetting across diverse tasks and model scales.
\end{itemize}

\section{Preliminaries \& Related Works}\label{sec:related_work}

% \subsection{Transformers}
% Transformers can be understood from two key submodules: multi-head attention (MHA) and feed-forward network (FFN). The MHA with $h$ parallel heads performs the attention function as follows:
% \begin{align}
%     \text{MHA}(X) = \text{Concat}(\text{head}_1, \dots, \text{head}_h) W_o ;  \quad \text{head}_i = \text{Softmax} \left( \frac{X W_{q_i} (X W_{k_i})^\top}{\sqrt{d_k}} \right) X W_{v_i},
% \end{align} where $W_o \in \mathbb{R}^{d \times d}$ is an output projection weight and $W_{q_i}, W_{k_i}, W_{v_i} \in \mathbb{R}^{d \times d_h}$ are query, key, and value projection weights for each head $i$. $d_h$ is typically set to $d/h$. FFN performs two linear transformations with a ReLU activation as follows:
% \begin{align}
%     \text{FFN}(X) = \text{ReLU}(XW_{f_1} + b_1)W_{f_2} + b_2,
% \end{align} where $W_{f_1} \in \mathbb{R}^{d \times d_m}$ and $W_{f_2} \in \mathbb{R}^{d_m \times d}$. These architectures enable a model to understand the language patterns and generate human-like texts in natural language processing.

\subsection{Low-Rank Adaptation and SVD}
LoRA~\citep{hu2021lora} suggests the low-rank update of the pre-trained weights by the product of two low-rank matrices. For $h=W_0x$, the modified forward pass becomes:
\begin{align}
    h=W_0x+ \Delta Wx =W_0x + BAx,
\end{align}
where $W_0,\Delta W \in \mathbb{R}^{d_1 \times d_2}$, $A \in \mathbb{R}^{r\times d_2}$ and $B \in \mathbb{R}^{d_1 \times r}$ with $r \ll \{d_1, d_2\}$. $A$ is initialized with a random Gaussian initialization and $B$ with zero, so $\Delta W=BA$ is initially zero at the beginning of training. After fine-tuning, the learnable adapter $\Delta W$ can be integrated into the pre-trained weight $W$ without modifying the original model architecture or adding any additional inference overhead.

\paragraph{SVD-based LoRA.}
% Recent studies have explicitly decomposed the network parameters using SVD to initialize adapters with a subset of components.
PiSSA~\citep{meng2024pissa} assumes that the principal components hold important information, decomposing the network parameters into principal and residual components using explicit SVD. Then the residual components are frozen, while the adapter is initialized with the principal components and directly updated. 
Conversely, MiLoRA~\citep{wang2025milora} proposes directly modifying the minor components of the pre-trained networks, assuming they are noisy and less important, in order to better preserve the pre-trained knowledge.
CorDA~\citep{yang2024corda} performs SVD on $W_0C$ with input activation covariance $C$ from  a small calibration set, to obtain a context-aligned low-rank decomposition. 
LoRA-GA~\citep{wang2024loraga} computes full-parameter gradients on a small sample set and applies SVD to obtain a low-rank subspace, initializing LoRA adapters so that their updates approximate the full-gradient direction.
Unlike existing initialization-based method, AdaLoRA~\citep{zhang2023adalora} adopt parameterized SVD to dynamically adjusts the rank for LoRA layer based on a sensitivity-driven importance score. They focus on pruning the number of ranks to meet a predefined budget using heuristic importance scores.
Despite their architectural differences, however, existing methods do not explicitly provide a principled way to determine where task-specific adaptation should occur or how pretrained knowledge should be preserved, nor do they control the spectral dynamics of the updates that govern this process.

% \paragraph{LoRA with \textit{parameterized} SVD.}
% \RED{Recent studies design LoRA based on the parameterized SVD, which differ in their specific design strategies.}
% \RED{SORSA}
% LoRA$^2$~\citep{zhang2024lora2} uses the twice-nested parameterized SVD to iteratively project the token representations onto \textit{mutually orthogonal planes}.
% Mo-SARA~\citep{gu2024sara} initializes the singular vectors with principal components of the pre-trained network parameters. They freeze the singular vectors and fine-tune only the randomly initialized singular values under the \textit{same eigenvector mappings} with the pre-trained network parameters. 

\subsection{Catastrophic forgetting and LoRA}
    
Catastrophic forgetting refers a well-known issue in the field of deep learning~\citep{mccloskey1989catastrophic,french1999catastrophic,kirkpatrick2017overcoming}, where the models forget previously acquired knowledge during adaptation to new tasks. To address this challenge, recent studies have proposed various approaches, including knowledge distillation~\citep{li2017kdlearning,hou2019kdlearning}, rehearsal~\citep{riemer2018rehearsallearning,yang2023rehearsalneural} and dynamic architectures~\citep{yan2021dadynamically}. This issue is particularly severe in LLMs, which learn extensive world knowledge through the pre-training process on massive datasets. During the fine-tuning process, where task-specific information is learned based on this world knowledge, forgetting the pre-trained knowledge can significantly undermine the stability and scalability of the models. Catastrophic forgetting has also been observed in parameter-efficient fine-tuning methods, including LoRA, prompting recent studies to propose various approaches to mitigate this issue~\citep{wang2025milora,yang2024corda,ren2024analyzing,yang2024moral,dou2024loramoe}.

\section{Proposed Method}\label{sec:proposed_method}

\subsection{Low-Rank Adaptation with Spectral Clipping}\label{subsec:method}

Our analysis shows that the minor singular components of the pre-trained weights exhibit a lower degree of task alignment and therefore require substantially more task-specific adaptation than the principal components. Furthermore, we empirically and theoretically observe that the growth of adapter singular values during fine-tuning weakens the pre-trained prior and consequently inducing catastrophic forgetting.
Therefore, existing SVD-based approaches such as \citep{meng2024pissa, wang2025milora, wang2024loraga, yang2024corda,zhang2023adalora} operate in the spectral domain by initializing adapters with SVD-derived components at the weight, activation, or gradient level, or by controlling the number of rank via parameterized SVD; however, the subsequent adaptation dynamics along the singular spectrum remain uncontrolled.
Our goal is to ensure that, regardless of the initialization scheme, the adapter primarily learns within the minor spectral subspace of the pre-trained weights—where task-specific adaptation is needed—while avoiding distortion of the principal components, and our approach explicitly controls the spectrum during fine-tuning.

To realize this objective, we introduce spectral clipping on the adapter singular values.
Singular Value Clipping (SVC) has been widely applied in various domains, including GANs and CNNs~\citep{saito2017gansvc,senderovich2022cnnsvc}. By controlling the spectral norm of each layer, SVC preserves the Lipschitz continuity of the network, leading to improved training stability, mitigation of exploding and vanishing gradients, and enhanced generalization and robustness.
However, applying explicit SVD on adapter and clipping at every training step is computationally infeasible for LLM-scale parameters. We therefore adopt a parameterized SVD framework, which enables learning directly in the spectral domain while allowing the singular values of the adapter to be clipped within a controlled range in an efficient manner.
Formally, we parameterize the adapter as a low-rank matrix whose singular values are clipped with respect to an upper bound derived from the spectral distribution of the pre-trained weights.
\begin{align}\label{eq:hilora}
W = W_0 + \Delta W = W_0 + U \Sigma  V^{\top},
\end{align} where $U \in \mathbb{R}^{d_1 \times r}$, $V \in \mathbb{R}^{d_2 \times r}$ are parameterized left and right singular vectors, respectively, and $\Sigma \in \mathbb{R}^{r}$ contains the parameterized singular values $\{s_n\}_{1\leq n \leq r}$. At the start of fine-tuning, $U,V$ are initialized with random $r$ singular vectors of $W_0$, or $U$ is initialized with zero and $V$ with a random Gaussian. Note that SVD on $W_0$ is performed only once before fine-tuning, and the actual operation does not involve any explicit decomposition or reconstruction of $W_0$ during the fine-tuning process. 

\begin{figure}[t]
    \centering
   {\includegraphics[width=0.95\columnwidth]{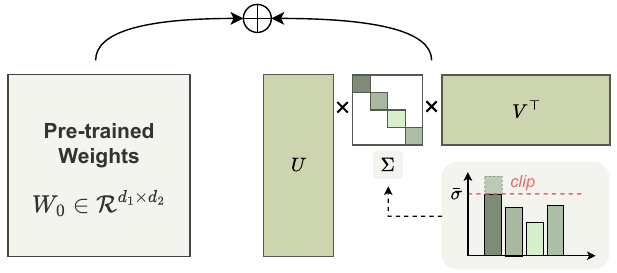}}
    \caption{The architecture of \textbf{SCLoRA}}
    \label{fig:overall_architecture}
\end{figure}

A natural question that follows is: \textit{how should we select an appropriate upper bound on the singular values?}
As illustrated in~\Cref{fig:intro} (a), the principal components of the pretrained network are already aligned with the new task, whereas the minor components exhibit a lower degree of task alignment and therefore require more substantial adaptation.
This observation suggests introducing spectral-aware constraints that guide learning toward the regions of the spectrum where adaptation is needed.
However, using a fixed scalar upper bound is generally arbitrary and sensitive to model scale and layer dimensionality. Instead, defining the upper bound in terms of the quantiles of the pretrained spectral distribution provides a scale-independent and distribution-aware constraint.
Accordingly, we impose a constraint that limits the growth of the adapter singular values while respecting the spectral structure of the pretrained weights. We realize this constraint through \textit{quantile}-based clipping, which restricts the adapter updates to remain within a bounded region defined by the pretrained spectral distribution.
Formally, we denote the singular values of the pretrained weight matrix $W_0$ as ${\sigma_i(W_0)}_{i=1}^{d}$.
The spectral bound $\bar{\sigma}$ is then defined as the $q$-th quantile $Q_q(\cdot)$ of the pretrained spectral distribution:
\begin{align}
\bar{\sigma}
= Q_q({\sigma_i(W_0)}_{i=1}^{d}).
\end{align}
Then, the parameterized singular values are constrained by quantile-based spectral clipping as:
\begin{align}
s_n \leftarrow \min\bigl(\max(s_n, 0), \bar{\sigma}\bigr),
\end{align}
for $n \in \{1,\dots,r\}$. This formulation explicitly links the spectral bound of the adapter to the pretrained spectrum while remaining independent of the absolute parameter scale.
Thus, quantile-based spectral clipping establishes a distribution-aware trust region, enabling controlled adaptation in the regions that require greater modification. An ablation study on the effect of $q$ is in~\Cref{sec:sens_sigma_bar}.
To ensure that the singular values remain non-negative, we set the lower bound to zero in accordance with the definition of SVD.
Furthermore, to enforce the orthogonality of the singular vectors, i.e., $U^\top U = V V^\top = I$, we apply the following regularization term:
\begin{align}
    R(U,V)=\Vert U^\top U-I \Vert + \Vert VV^\top - I\Vert
\end{align} where $I \in \mathbb{R}^{r \times r}$ indicates an identity matrix. This regularization term is controlled by the orthogonal regularization coefficient  $\gamma$. We verify the orthogonality of the parameterized singular vectors in \Cref{app:orthogonality}. We present the training process in~\Cref{alg:train} of~\Cref{app:train_algorithm}.

% and report both analytical and empirical cost analyses — including wall-clock runtime, GPU memory usage— in~\Cref{app:complexity}.

\subsection{Theoretical analysis}\label{sec:theoretical_analysis}

% Following the standard Bayesian formulation of MAP estimation, the objective can be written as:
The optimization of deep learning models, including LoRA, can be interpreted as a Maximum A Posteriori (MAP) estimation of the network parameters $\theta$ on the training data. In transfer learning, the model is first pre-trained on the dataset $\mathcal{D}_A$ and then fine-tuned on the dataset $\mathcal{D}_B$. 
Following the standard Bayesian formulation of MAP estimation, the posterior to be maximized in the MAP estimation is formulated as:
\begin{align}
p(\theta | \mathcal{D}_A, \mathcal{D}_B)&=\frac{p(\mathcal{D}_B|\theta, \mathcal{D}_A)p(\theta|\mathcal{D}_A)}{p(\mathcal{D}_B|\mathcal{D}_A)}\notag\\&=\frac{p(\mathcal{D}_B|\theta)p(\theta|\mathcal{D}_A)}{p(\mathcal{D}_B|\mathcal{D}_A)}.  
\end{align}

Taking a logarithm of the posterior, the MAP objective becomes:
\begin{align}
    \theta^*&= \underset{\theta}{\mathrm{argmax}} \log p(\theta|\mathcal{D}_A,\mathcal{D}_B) \notag\\&= \underset{\theta}{\mathrm{argmax}} \big[ \log p(\mathcal{D}_B|\theta) + \log p(\theta|\mathcal{D}_A)\big].
\end{align}

% \begin{align}
% \arg\max_{\theta}\Big[\log p(\mathcal{D}_B|\theta)+\log p(\theta|\mathcal{D}_A)\Big].
% \end{align}
The first term corresponds to the likelihood of $\mathcal{D}_B$ and is expressed as the loss function for the fine-tuning task, while the second term represents the prior over the parameters induced by $\mathcal{D}_A$.
During fine-tuning, we treat the posterior from the pre-trained task, $p(\theta|\mathcal{D}_A)$, as the prior for the fine-tuning task. However, since the true posterior is intractable, it can be expressed as a function $f(\theta)$ and approximated around the pre-trained mode $\theta_{0}$ using the Laplace approximation, a well-established method in Bayesian deep learning for handling intractable posteriors~\citep{kirkpatrick2017overcoming,ritter2018online,wang2021afec,matena2022merging,gawlikowski2023survey} as:
\begin{align}\label{eq:approx_posterior}
\log \hat{p}(\theta|\mathcal{D}_A) = f(\theta_{0})-
\frac{1}{2} (\theta - \theta_{0})^{\top} F (\theta - \theta_{0}),
\end{align}
where $F$ denotes the Fisher information matrix. The following theorem shows that this approximated prior is upper-bounded by the singular values of the difference between the pre-trained and fine-tuned parameters.

\begin{theorem}\label{thm:map}
Under the Laplace approximation, the approximated log prior probability satisfies:
\begin{align}
\log \hat{p}(\theta \mid \mathcal{D}_A) \le f(\theta_{0}) - \lambda_{\min}(F)\sqrt{\sum_{n=1}^{r}\sigma_{n}^{2}},
\end{align}
where $\lambda_{\min}(F)$ denotes the smallest eigenvalue of $F$,
and $\sigma_{n}$ is the $n$-th singular value of $\theta - \theta_{0}$.
\end{theorem}

The proof is described in \Cref{app:proof_map}. It is worth noting that the negligibility of the higher–order terms in the Laplace approximation is supported by a well-established body of theoretical results in Bayesian asymptotics \citep{kass1990validity}, where the omitted error term vanishes as the posterior concentrates around $\theta_{0}$. We provide further discussion of this approximation error in~\Cref{app:laplace_note}.

According to~\Cref{thm:map}, $\log \hat{p}(\theta|\mathcal{D}_A)$ is upper bounded by the singular values of the parameter difference, which is adapter in LoRA. Specifically, larger singular values of the adapter lead to a decrease in the posterior from the pre-training task, resulting in the loss of pre-trained knowledge, the phenomenon called \textit{catastrophic forgetting}. The catastrophic forgetting problem undermines the strengths of pre-trained models and hinders their adaptability to new tasks, making it crucial to address in order to maintain scalability and reliability in transfer learning~\citep{wang2025milora,yang2024corda,ren2024analyzing,yang2024moral,dou2024loramoe}.
In stochastic optimization, however, the spectral norm of weight matrices tends to grow rapidly. The following proposition from \citet{zhai2023stabilizing} establishes a lower bound on the spectral norm of the ideal update.

\begin{table*}[t]
\centering
\small
\setlength{\tabcolsep}{6.5pt}

\resizebox{0.9\textwidth}{!}{%
\begin{tabular}{l | c |  ccccccccc} \toprule
Method  & \# Params & MNLI  & SST-2  & CoLA  & QQP  & QNLI  & RTE  & MRPC  & STS-B & Avg. \\ \midrule
\rowcolor{gray!10}\multicolumn{11}{c}{\textit{Model: RoBERTa\textsubscript{base}}}  \\
\addlinespace[1.4pt]
LoRA  & 1.33M & 87.93 & 94.80 & 64.49 & 90.94 & 92.73 & 80.39 & 89.05 & 90.87 & 86.40 \\
   AdaLoRA & 1.27M & 87.90 & 94.80 & 62.41 & 90.16 & 92.67 & 81.59 & 89.38 & 91.08 & 86.25 \\
  PiSSA & 1.33M  & \textbf{87.95} & 94.53 & 64.66 & 90.97 & 92.53 & 79.18 & 89.79 & 90.96 & 86.32 \\
   % &  rsLoRA & 1.33M  & 85.26 & 92.35 & 65.17 & 70.76 & 92.48 & 79.54 & 89.05 & 90.88 & 83.19 \\
  MiLoRA & 1.33M  & 87.88 & 94.69 & 64.31 & \textbf{91.02} & 92.96 & 81.35 & 89.30 & 90.96 & 86.56 \\
  LoRA+ & 1.33M  & 86.96 & 93.92 & 63.32 & 90.69 & 92.77 & 81.59 & 88.97 & 90.84 & 86.13 \\
  LoRA-GA & 1.33M  &85.18 & 93.16 & 62.09 & 88.57 & 91.64 & 76.77 & 89.54 & 90.87 & 84.73 \\
CorDA & 1.33M  & 86.28 & 94.38 & 62.43 & 90.14 & 92.36 & 78.10 & 89.87 & 90.68 & 85.53 \\
   DoRA  & 1.41M & 87.81 & 95.11 & 64.23 & 90.65 & 92.93 & 81.35 & 89.54 & 91.01 & 86.58 \\
  \rowcolor{blue!10} \textbf{SCLoRA}& 1.33M  & \textbf{87.95} & \textbf{95.37} & \textbf{64.79} & 90.87 & \textbf{93.09} & \textbf{83.15} & \textbf{90.32} & \textbf{91.22} & \textbf{87.08} \\ 
\midrule
\rowcolor{gray!10}\multicolumn{11}{c}{\textit{Model: DeBERTaV3\textsubscript{base}}}  \\
\addlinespace[1.4pt]
    % Full FT      & 184M & 89.90    & 95.63    & 69.19 & 92.40    & 94.03 & 83.75 & 89.46 & 91.60 & 88.09 \\ \midrule
   BitFit       & 0.10M & 89.37    & 94.84    & 66.96 & 88.41    & 92.24 & 78.70 & 87.75 & 91.35 & 86.02 \\
     HAdapter     & 1.22M & 90.13   & 95.53    & 68.64 & 91.91    & 94.11 & 84.48 & 89.95 & 91.48 & 88.12 \\
   PAdapter     & 1.18M & 90.33   & 95.61    & 68.77 & 92.04    & 94.29 & 85.20 & 89.46 & 91.54 & 88.24 \\
  LoRA\textsubscript{r=8} & 1.33M & 90.65  & 94.95    & 69.82 & 91.99    & 93.87 & 85.20 & 89.95 & 91.60 & 88.34 \\
    AdaLoRA & 1.27M & \textbf{90.76} & 96.10 & 71.45
    & 92.23 & 94.55 & 88.09 & 90.69 & 91.84 & 89.31 \\
 \rowcolor{blue!10}  \textbf{SCLoRA}  &  1.33M &  90.36 & \textbf{96.33}  & \textbf{71.49} & \textbf{92.33}  & \textbf{94.57}  & \textbf{89.41}  &  \textbf{91.58} & \textbf{92.19} & \textbf{89.78} \\
 \midrule
     HAdapter     & 0.61M & 90.12    & 95.30    & 67.87 & 91.65     & 93.76 & 85.56 & 89.22 & 91.30 & 87.93 \\
    PAdapter     & 0.60M & 90.15    & 95.53    & 69.48 & 91.62    & 93.98 & 84.12 & 89.22 & 91.52 & 88.04 \\
      HAdapter     & 0.31M & 90.10    & 95.41    & 67.65 & 91.54    & 93.52 & 83.39 & 89.25 & 91.31 & 87.60 \\
   PAdapter     & 0.30M & 89.89    & 94.72    & 69.06 & 91.40    & 93.87 & 84.48 & 89.71 & 91.38 & 87.90 \\
  LoRA\textsubscript{r=2} & 0.33M & 90.30  & 94.95    & 68.71 & 91.61     & 94.03 & 85.56 & 89.71 & 91.68 & 88.15 \\
    AdaLoRA & 0.32M & \textbf{90.66} & 95.80 & 70.04
    & 91.78 & 94.49 & 87.36 & 90.44 & 91.63 & 88.86 \\
 \rowcolor{blue!10} \textbf{SCLoRA}  & 0.33M & \textbf{90.66} & \textbf{96.18} & \textbf{71.83} & \textbf{91.82} & \textbf{94.50} & \textbf{89.89} & \textbf{91.83} & \textbf{92.00} & \textbf{89.84} \\
    \bottomrule
    \end{tabular}
    }
    \caption{Comparison of various methods with RoBERTa\textsubscript{base} and DeBERTaV3\textsubscript{base} on GLUE tasks with different random seeds. Results with standard deviations are in \Cref{app:main_nlu_std}.}
    \label{tab:main_nlu}

\end{table*}

\begin{proposition}\label{pro:spectral_norm}
    From~\citet{zhai2023stabilizing}, it holds:
    \begin{align}\label{eq:spectral_norm}
        \Vert \Delta \Vert \geq \sqrt{d} \sqrt{1-\frac{1}{d^2}\sum_{i,j=1}^d\frac{\omega^2_{i,j}}{\mu^2{i,j}+\omega^2_{i,j}}}.
    \end{align}
\end{proposition}
The noise second moment $\omega^2$ is typically in the order of $\mu^2$. Hence, \Cref{pro:spectral_norm} indicates that the spectral norm of the ideal update should be large, growing linearly with $\sqrt{d}$. Moreover, for large batch sizes we would have $\omega^2\ll 1$, resulting in $\Vert \Delta \Vert \sim \sqrt{d}$. 
Proposition~\ref{pro:spectral_norm} from~\citet{zhai2023stabilizing} demonstrates that this growth becomes more pronounced in high-dimensional settings when adaptive optimizers are used. 
Therefore, in general probabilistic optimization for transfer learning, including LoRA, the spectral norm is implicitly driven toward larger values during fine-tuning. 
This increase in the largest singular value reduces the posterior probability of the pre-trained knowledge under the MAP formulation, thereby leading to catastrophic forgetting.
Empirically, a consistent spectrum–forgetting trend is observed in our experiments: In~\Cref{sec:spectral_analysis}, the singular values of the LoRA adapter grow during fine-tuning and this growth is accompanied by degradation of the pre-training performance, indicating that LoRA can be vulnerable to catastrophic forgetting in practice.

\section{Complexity analysis}

In this section, we analyze the complexity of SCLoRA. Empirical runtime and GPU usage are further detailed in \Cref{app:empirical_complexity}.

\paragraph{Parameter Complexity.}
LoRA updates $(d_1+d_2)r$ parameters per layer. \textbf{SCLoRA} introduces a minimal overhead of only $r$ parameters for the singular values, resulting in $(d_1+d_2)r + r$ trainable parameters. This successfully preserves the inherent parameter efficiency of the baseline method.

\paragraph{Computational Complexity.}
During training, standard LoRA requires $\mathcal{O}(d_1 d_2 r)$ operations. \textbf{SCLoRA} operates with a slightly higher complexity of $\mathcal{O}(d_1 d_2 r + (d_1+d_2)r + r^3)$, consistent with other parameterized SVD-based approaches. Due to the strict low-rank constraint, this additional overhead is mathematically negligible. During inference, standard LoRA scales as $\mathcal{O}(d_1 d_2 r)$. By scaling the singular vectors with parameterized singular values, \textbf{SCLoRA} achieves an inference complexity of $\mathcal{O}(d_1 d_2 r + d_2 r)$, which remains highly comparable to the baseline.

\section{Experiments}\label{sec:experiments}
We empirically verify that \textbf{SCLoRA} efficiently adapts to the task and improves the performance.

\subsection{Natural Language Understanding}

\paragraph{Experimental setup.}
Following \citet{hu2021lora} and \citet{zhang2023adalora}, we fine-tune the pre-trained RoBERTa\textsubscript{base} and DeBERTa\textsubscript{base} models on the General Language Understanding Evaluation (GLUE) benchmark~\citep{wang2018glue}, setting the rank to $r=8$ for RoBERTa\textsubscript{base} and $r \in \{2, 8\}$ for DeBERTa\textsubscript{base}.
We report the Matthews correlation for CoLA, Spearman’s correlation for STS-B, and accuracy for all other tasks. Detailed descriptions are provided in \Cref{app:setup_nlu}.

\paragraph{Main results.}
In \Cref{tab:main_nlu}, \textbf{SCLoRA} outperforms various baselines on average. In particular, MiLoRA, which is closely related to our method, fails to achieve optimal performance due to information loss and uncontrolled fine-tuning. By contrast, \textbf{SCLoRA} efficiently adapts to the new task by injecting spectrally clipped singular components.

\begin{table*}[t]
\centering
\small
\resizebox{0.9\textwidth}{!}{%
\begin{tabular}{l|cccc|cccc}
\toprule
\multirow{2}{*}{Method} & \multicolumn{4}{c|}{SQuADv1.1} & \multicolumn{4}{c}{SQuADv2.0} \\ \cmidrule{2-9}
                  & \multicolumn{1}{c}{0.08\%} & \multicolumn{1}{c}{0.16\%} & \multicolumn{1}{c}{0.32\%} & \multicolumn{1}{c|}{0.65\%} & \multicolumn{1}{c}{0.08\%} & \multicolumn{1}{c}{0.16\%} & \multicolumn{1}{c}{0.32\%} & \multicolumn{1}{c}{0.65\%} \\ \midrule  
Full FT          & \multicolumn{4}{c|}{86.0 / 92.7}    & \multicolumn{4}{c}{85.4 / 88.4}    \\ \midrule
HAdapter         & 84.4/91.5    & 85.3/92.1    & 86.1/92.7    & 86.7/92.9   & 83.4/86.6    & 84.3/87.3    & 84.9/87.9    & 85.4/88.3  \\
PAdapter         & 84.4/91.7    & 85.9/92.5    & 86.2/92.8    & 86.6/93.0   & 84.2/87.2    & 84.5/87.6    & 84.9/87.8    & 84.5/87.5  \\
LoRA             & 86.4/92.8    & 86.6/92.9    & 86.7/93.1    & 86.7/93.1& 84.7/87.5    & 83.6/86.7    & 84.5/87.4    & 85.0/88.0   \\
AdaLoRA          & 87.2/93.4 & 87.5/93.6 & 87.5/93.7 & 87.6/93.7 &\textbf{85.6/88.7} & 85.7/88.8 & 85.5/88.6 & 86.0/88.9  \\ 
% PiSSA          & &&&&&&& \\ 
% MiLoRA          & &&&&&&&  \\ 
\rowcolor{blue!10}\textbf{SCLoRA}         & \textbf{87.6/93.6}&  \textbf{88.1/93.9} & \textbf{88.2/94.1}  & \textbf{88.6/94.3} & 85.3/88.3 & \textbf{86.0}/\textbf{89.0} & \textbf{86.0}/\textbf{88.8} & \textbf{86.2/89.0} \\ \bottomrule
% \textbf{SCLoRA}         &  \textbf{88.1/93.9} & \textbf{88.0/93.9} & \textbf{88.7/94.3} & \RED{\textbf{88.9/94.4}} & \textbf{85.9}/\textbf{88.7} & \textbf{85.9}/\textbf{88.9} & \RED{\textbf{86.1/89.1}} & \textbf{86.2/89.1}\\ \bottomrule

    \end{tabular}
    }
\caption{Comparison of various methods with DeBERTaV3\textsubscript{base} on SQuAD datasets}
\label{tab:squad}
\end{table*}

\begin{table*}[t]
    \centering
    \small
    \resizebox{0.95\textwidth}{!}{%
    \begin{tabular}{l l c c c c c c c c c c}\toprule
        % Model & Method & \#Params (\%) & BoolQ & PIQA & SIQA & HellaSwag & Wino. & ARC-e & ARC-c & OBQA & Avg. \\
        Model & Method & \#Params (\%) & BoolQ & PIQA & SIQA & Hella. & Wino. & ARC-e & ARC-c & OBQA & Avg. \\
        \midrule
        ChatGPT & - & - & 73.1 & 85.4 & 68.5 & 78.5 & 66.1 & 89.8 & 79.9 & 74.8 & 77.0 \\
        \midrule
        \multirow{5}{*}{LLaMA-7B} & Prefix & 0.11 & 64.3 & 76.8 & 73.9 & 42.1 & 72.1 & 72.9 & 54.0 & 60.6 & 64.6 \\
        % & Series & 0.99 & 63.0 & 79.2 & 76.3 & 67.9 & 75.7 & 74.5 & 57.1 & 72.4 & 70.8 \\
        % & Parallel & 3.54 & 67.9 & 76.4 & 78.8 & 69.8 & 78.9 & 73.7 & 57.3 & 75.2 & 72.2 \\
        & LoRA & 0.83 & 68.9 & 80.7 & 77.4 & 78.1 & 78.8 & 77.8 & 61.3 & 74.8 & 74.7 \\
        & DoRA$^\dagger$ & 0.43 & 70.0 & 82.6 & 79.7 & 83.2 & 80.6 & 80.6 & 65.4 & 77.6 & 77.5 \\
        & DoRA & 0.84 & 69.7 & 83.4 & 78.6 & 87.2 & 81.0 & 81.9 & 66.2 & 79.2 & 78.4 \\
        \rowcolor{blue!10}\cellcolor{white} & \textbf{SCLoRA} & 0.83 & 70.9 & 83.8 & 80.1 & 88.3 & 81.5 & 82.8 & 66.5 & 81.0 & \textbf{79.4} \\
        % \rowcolor{blue!10}\cellcolor{white} & \textbf{SCLoRA} & 0.83 & 70.9 & 83.8 & 80.1 & 88.3 & 81.5 & 82.8 & 66.5 & 81.0 & \textbf{79.4} \\
        \midrule
        \multirow{6}{*}{LLaMA2-7B} & LoRA & 0.83 & 69.8 & 79.9 & 79.5 & 83.6 & 82.6 & 79.8 & 64.7 & 81.0 & 77.6 \\
        & PiSSA	& 0.83 & 67.6 & 78.1 & 78.4 & 76.6 & 78.0 & 75.8 & 60.2 & 75.6 & 73.8\\
        & MiLoRA & 0.83  & 67.6 & 83.8	 & 80.1 & 88.2 & 82.0 & 82.8 & 68.8 & 80.6 & 79.2 \\
        & DoRA\textsuperscript{\dag} & 0.43 & 72.0 & 83.1 & 79.9 & 89.1 & 83.0 & 84.5 & 71.0 & 81.2 & 80.5 \\
        & DoRA & 0.84 & 71.8 & 83.7 & 76.0 & 89.1 & 82.6 &83.7 & 68.2 & 82.4 & 79.7 \\
        \rowcolor{blue!10}\cellcolor{white}  & \textbf{SCLoRA} & 0.83 & 69.3 & 84.2 & 80.4 & 87.1 & 85.4 & 85.6 & 72.7 & 83.6 & \textbf{81.0} \\
                           
                           \bottomrule
    \end{tabular}
    }
    \caption{Comparison of various methods with LLaMA on commonsense reasoning tasks}
    \label{tab:cs_llama}
\end{table*}

\subsection{Question Answering}
\paragraph{Experimental setup.}
Following~\citet{zhang2023adalora}, we fine-tune DeBERTaV3\textsubscript{base}~\citep{he2021debertav3} under four different budgets: 0.08\%, 0.16\%, 0.32\%, and 0.65\% of the total pre-trained parameters. We evaluate \textbf{SCLoRA} on two question-answering (QA) benchmarks—SQuAD v1.1 \citep{rajpurkar2016squad} and SQuAD v2.0 \citep{rajpurkar2018squadv2}—using the Exact Match (EM) and F1 metrics. Detailed descriptions are in \Cref{app:setup_qa}. 

\paragraph{Main results.}
% \Cref{tab:squad} illustrates that \textbf{SCLoRA} delivers performance comparable to or better than existing baselines in most parameter budgets, suggesting that spectral clipping of the adapter is effective for QA adaptation under various budgets.
\Cref{tab:squad} illustrates that \textbf{SCLoRA} delivers performance comparable to or better than existing baselines in most parameter budgets, suggesting that spectral clipping of the adapter is effective for QA adaptation under various budgets. Furthermore, it outperforms strong baselines like AdaLoRA in most cases, demonstrating the robustness of our spectral constraint approach.

\subsection{Commonsense reasoning}

\paragraph{Experimental setup.}

We evaluate \textbf{SCLoRA} on the commonsense reasoning tasks. Following~\citet{hu2023llmadapters}, we amalgamate the training datasets from all 8 tasks to create the final training dataset and evaluate with individual testing for each task. We fine-tune LLaMA-7B~\citep{touvron2023llama} and LLaMA2-7B~\citep{touvron2023llama2}. The detailed descriptions are provided in \Cref{app:setup_cr}.

\paragraph{Main results.}

\Cref{tab:cs_llama} summarizes the comparison of various fine-tuning methods on commonsense reasoning benchmarks. \textbf{SCLoRA} demonstrates superior performance across both LLaMA-7B and LLaMA2-7B architectures with average accuracy of 79.4 and 81.0. Importantly, SCLoRA not only yields substantial improvements over standard LoRA (+4.7\% and +3.4\% on average), but it also outperforms recent advanced methods such as DoRA and MiLoRA. These results highlights the effectiveness of spectral clipping on adapters for reasoning performance even in larger models. 
% \Cref{tab:cs_llama} shows that \textbf{SCLoRA} outperforms other methods on average, highlighting the effectiveness of spectral clipping on adapters for reasoning performance even in larger models. 

\section{Mitigation of catastrophic forgetting}\label{sec:discussion}
\begin{table*}[t]
\centering
% \small
\label{tab:combined_comparison}
\resizebox{\linewidth}{!}{%
% 컬럼 정의 중앙에 세로선(|)과 양옆 15pt 여백 추가
\begin{tabular}{llccc ccc | llccc ccc}
\toprule

% 첫 번째 헤더 행 (세로선이 끊기지 않도록 multicolumn에도 | 추가)
\rowcolor{gray!10}
\multicolumn{8}{c|}{\textit{Model: RoBERTa\textsubscript{base}}} & \multicolumn{8}{c}{\textit{Fine-tuning tasks: Commonsense reasoning}} \\
\addlinespace[1.4pt]

% 두 번째 헤더 행 (SST-2가 끝나는 지점에도 | 추가)
\multirow{2}{*}{Task} & \multirow{2}{*}{Method} & \multicolumn{3}{c}{MRPC} & \multicolumn{3}{c|}{SST-2} & \multirow{2}{*}{Task} & \multirow{2}{*}{Method} & \multicolumn{3}{c}{LLaMA-7B} & \multicolumn{3}{c}{LLaMA-2-7B} \\
\cmidrule(lr){3-5} \cmidrule(lr){6-8} \cmidrule(lr){11-13} \cmidrule(lr){14-16}

% 세 번째 헤더 행
 & & $\Vert \cdot \Vert_2$ & Acc.\textsubscript{ft} & Acc.\textsubscript{pt} & $\Vert \cdot \Vert_2$ & Acc.\textsubscript{ft} & Acc.\textsubscript{pt} & & & $\Vert \cdot \Vert_2$ & Acc.\textsubscript{ft} & PPL.\textsubscript{pt} & $\Vert \cdot \Vert_2$ & Acc.\textsubscript{ft} & PPL.\textsubscript{pt} \\
\midrule

% 첫 번째 블록 (BC & PG19)
\multirow{3}{*}{BC} 
& Pre-trained & - & - & 61.64 & - & - & 61.64 & \multirow{3}{*}{PG19} & Pre-trained & & - & 6.66 & & - & 6.61 \\
& LoRA        & 3.39 & 89.05 & 3.77 & 12.12 & 94.80 & 32.35 & & LoRA        & 13.16 & 74.7 & 8.69 & 21.46 & 77.6 & 11.36 \\
\rowcolor{blue!10}\cellcolor{white} & \textbf{SCLoRA} & 0.94 & \textbf{90.32} & \textbf{32.00} & 0.45 & \textbf{95.37} & \textbf{51.29} & \cellcolor{white} & \textbf{SCLoRA} & 5.09 & \textbf{79.4} & \textbf{7.79} & 8.70 & \textbf{81.0} & \textbf{8.29} \\
\midrule

% 두 번째 블록 (OWT & C4-en)
\multirow{3}{*}{OWT} 
& Pre-trained & & - & 68.21 & & - & 68.21 & \multirow{3}{*}{C4\textsubscript{en}} & Pre-trained & & - & 7.58 & & - & 7.61 \\
& LoRA        & 3.39 & 89.05 & 18.36 & 12.12 & 94.80 & 54.33 & & LoRA        & 13.16 & 74.7 & 10.54 & 21.46 & 77.6 & 15.26 \\
\rowcolor{blue!10}\cellcolor{white} & \textbf{SCLoRA} & 0.94 & \textbf{90.32} & \textbf{47.27} & 0.45 & \textbf{95.37} & \textbf{65.67} & \cellcolor{white} & \textbf{SCLoRA} & 5.09 & \textbf{79.4} & \textbf{9.53} & 8.70 & \textbf{81.0} & \textbf{10.71} \\

\bottomrule
\end{tabular}
}
\caption{Comparison on catastrophic forgetting}
\label{tab:cf}
\end{table*}

%%%%%%%%%%%%%%%%%%%%%%%%%%%%%%%%%%%%

This section validates that \textbf{SCLoRA} successfully mitigates \textit{catastrophic forgetting}.
\Cref{tab:cf} reports the spectral norm of adapter, fine-tuning accuracy ($\uparrow$), and pre-trained task performance measured by accuracy ($\uparrow$) or perplexity ($\downarrow$), providing a quantitative comparison of catastrophic forgetting across different methods.

The table on the left shows the results of fine-tuning the RoBERTa\textsubscript{base} model using the MRPC and SST-2 datasets, respectively. After fine-tuning, the model was evaluated on the pre-trained datasets BookCorpus (BC) and OpenWebText (OWT).
When applying LoRA to the MRPC, the downstream performance is 89.05; however, the spectral norm increases to 3.39, and the accuracy on the BC drops from 61.64 to 3.77. Similarly, the accuracy on the OWT drops substantially from 68.21 to 18.36. In contrast, under the same conditions, the spectral norm of  \textbf{SCLoRA} is 0.94, which is approximately 27\% of that of LoRA, while achieving accuracies of 32.00 and 47.27 on the BC and OWT, respectively.
Similarly, in the right table, we evaluate LLaMA-7B and LLaMA2-7B fine-tuned on commonsense reasoning tasks, and report perplexity on PG19 and C4\textsubscript{en} datasets. 
Compared to LoRA, \textbf{SCLoRA} suppresses the increase in the spectral norm to approximately 40\% while mitigating the increase in perplexity on the pre-training corpus, successfully preserving the knowledge of the original model.
Overall, these results support that uncontrolled spectral norm growth is a key factor underlying catastrophic forgetting, and that constraining adaptation by injecting singular components with spectral clipping during fine-tuning enables stable task adaptation while effectively preserving pre-trained knowledge. Additional experiments, including the fine-tuning evolution of catastrophic forgetting and comparisons with additional baselines, are reported in~\Cref{app:sec_cf}.

\section{Additional Studies}
In \Cref{app:additional_studies}, we conduct sensitivity analyses of $\gamma$, and $r$, and ablation studies of layer-wise $q$ selection and automatic $\bar{\sigma}$.

\begin{figure}[t]
    \centering
    \subfigure[RTE]{\includegraphics[width=0.49\columnwidth]{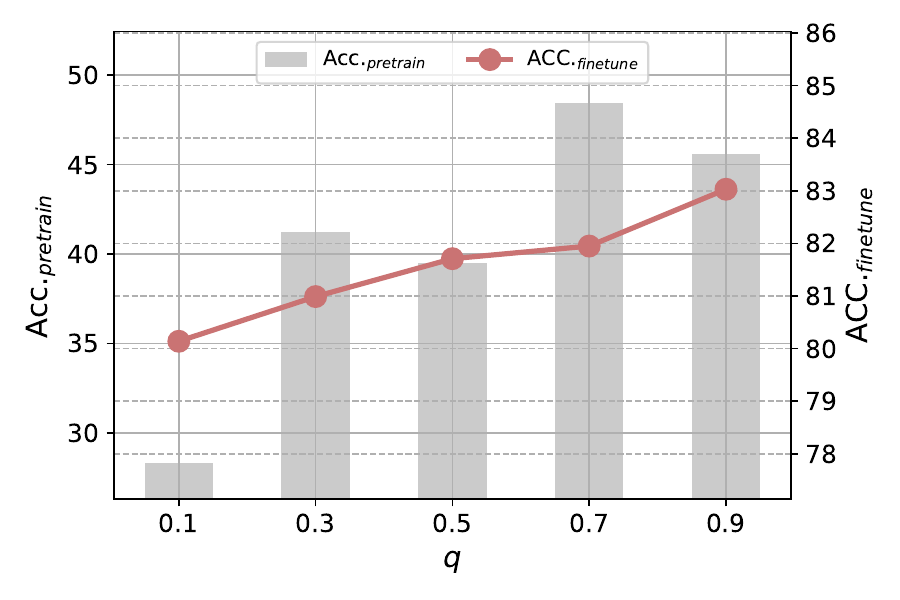}}
    \subfigure[MRPC]{\includegraphics[width=0.49\columnwidth]{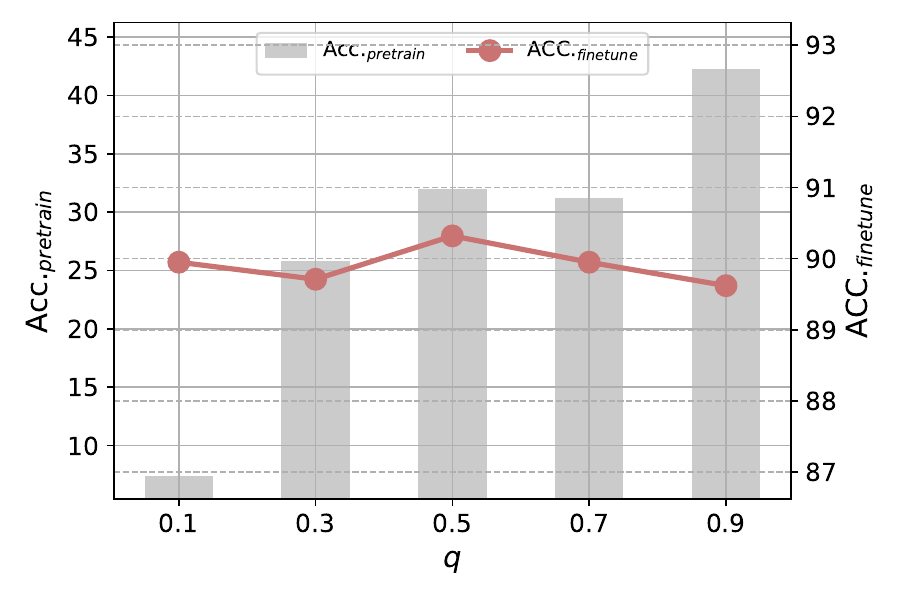}}
    \caption{Sensitivity on $q$}
    \label{fig:sens_quantile}
\end{figure}

\subsection{Sensitivity study on $q$}\label{sec:sens_sigma_bar}

In \Cref{fig:sens_quantile}, we conduct a sensitivity analysis of $q$, which represents the $q$-th quantile, measuring performance changes on both the pre-training and fine-tuning tasks. When $q = 0.1$, the pre-training performance drops, which is consistent with \Cref{thm:map} showing that excessively large adapter singular values weaken the pretrained prior and induce catastrophic forgetting. In contrast, once the quantile exceeds a moderate range, the pretrained performance rapidly recovers, indicating that spectral clipping is sufficient to preserve pretrained knowledge.
Meanwhile, the fine-tuning performance is less sensitive to the choice of $q$. For example, in MRPC, we observe a slight improvement at $q = 0.1$, suggesting that weak spectral constraints can amplify task-specific signals in low-resource settings. However, this gain comes at the cost of severe degradation in pre-training performance and is therefore unlikely to be desirable in practice.
Overall, we observe that when $q$ is set to a moderate range (e.g., $q \ge 0.3$), both pretrained retention and downstream performance remain stable. This indicates that extremely small quantile values should be avoided, whereas a broad moderate region yields consistently robust behavior across tasks.

\subsection{Ablation study on the injected components}

\begin{table}
    \small
    \centering
    \resizebox{0.85\linewidth}{!}{%
    \begin{tabular}{ l| cc | cc } \toprule
    \multirow{2}{*}{Model} & \multicolumn{2}{c|}{MRPC} & \multicolumn{2}{c}{SST-2}  \\       \cmidrule{2-3} \cmidrule{4-5}
                    &  Acc.\textsubscript{ft} & Acc.\textsubscript{pt} &  Acc.\textsubscript{ft} & Acc.\textsubscript{pt}  \\ \midrule
    Pre-trained    &  - & 61.64 & - & 61.64 \\ \midrule
    LoRA           & 89.05 & 3.77 & 94.81 & 32.35   \\
    LoRA\textsubscript{$UV^\top$}     & 89.22 & 3.12 & 94.75 &  39.39 \\
    LoRA\textsubscript{SVD}      & 89.62 & 17.57 & 95.03  &  49.65   \\
    LoRA\textsubscript{$\gamma$=0} & 90.20 & 28.31 & 94.84 & 45.74 \\
    \rowcolor{blue!10}\textbf{SCLoRA}         &  \textbf{90.32} & \textbf{32.00} & \textbf{95.37} & \textbf{51.29}   \\ \midrule
    \end{tabular}
    }
    \caption{Ablation on the injected components}
    \label{tab:abla_svd}
\end{table}

To analyze the influence of the injected components in \textbf{SCLoRA} on the performance of both pre-trained and fine-tuned knowledge, we conduct an ablation study on the following variants: i) `LoRA' refers to the traditional LoRA; ii) `LoRA\textsubscript{$UV^\top$}' applies orthogonal regularization to the singular vectors without the singular values; iii) `LoRA\textsubscript{SVD}' initializes the singular values as ones, allowing them to be learnable from LoRA\textsubscript{$UV^\top$}; iv) `LoRA\textsubscript{$\gamma$=0} applies singular value clipping without orthonormal regularization; and v) `\textbf{SCLoRA}' refers to the proposed method. We measure the accuracy on both the fine-tuning tasks with MRPC and SST-2, and the pre-trained task with BookCorpus dataset. As reported in \Cref{tab:abla_svd}, LoRA significantly sacrifices pre-training performance to adapt on new task. For the MRPC, accuracy on the pre-trained task drops from 61.64 to 3.77, while it achieves comparable accuracy on the fine-tuning task. LoRA\textsubscript{$UV^\top$} has limited expressiveness since its singular values are fixed, sacrificing either performance of pre-trained or fine-tuning task. LoRA\textsubscript{SVD} performs better due to its learnable singular values than LoRA\textsubscript{$UV^\top$}. LoRA\textsubscript{$\gamma$=0} shows promising performance; however, since it does not ensure the orthogonality of the singular vectors, \Cref{thm:map} cannot be applied, and thus it lacks a theoretical guarantee that catastrophic forgetting is consistently mitigated. Notably, \textbf{SCLoRA} learns the singular components with spectral clipping, ensuring both effective adaptation to the fine-tuning task and retention of pre-trained knowledge. For both datasets, \textbf{SCLoRA} achieves the best performance on both tasks.

\section{Conclusion}\label{sec:conclusion}

We propose a novel LoRA method, motivated by the following two rigorous analyses regarding the singular components of network parameters: i) We, for the first time, analyze the Fisher overlap of the pre-trained weights across the singular components. Specifically, the principal singular components of pre-trained weights can be reused for fine-tuning tasks, whereas the minor singular components require significant adaptation; ii) The growth of singular values in the adapters directly causes the phenomenon \textit{catastrophic forgetting}. 
From these analyses, we propose \textbf{SCLoRA}, which injects the parameterized singular components with spectral clipping.
Comprehensive experiments show that \textbf{SCLoRA} achieves strong performance on fine-tuning tasks and successfully retains the pre-trained knowledge. For future direction, the spectrum clipping of SCLoRA could be extended to the Mixture of Experts (MoE) architecture.
    
\section*{Limitation}
Although $q$ is introduced as a hyperparameter, the sensitivity analysis in~\Cref{sec:sens_sigma_bar} shows that once $q$ exceeds a moderate range, both fine-tuning and pre-trained performance remain stable. While small variations on performance may exist, the method is generally robust to the choice of $q$. Moreover, this stability is observed consistently across multiple tasks and model scales, suggesting the practical sensitivity of $q$ is limited.

\section*{Acknowledgments}
% Noseong Park is the corresponding author. This work was partly supported by the Institute for Information \& Communications Technology Planning \& Evaluation (IITP) grants funded by the Korean government (MSIT) (No. RS-202400457882, AI Research Hub Project), Samsung Electronics Co., Ltd. (No. G01240136, KAIST Semiconductor Research Fund (2nd)).
Noseong Park is the corresponding author. This work was partly supported by the Institute for Information \& Communications Technology Planning \& Evaluation (IITP) grants funded by the Korean government (MSIT) (No. RS-202400457882, AI Research Hub Project, 34\%; N10260110, AI Meta-Scientist, 33\%), Samsung Electronics Co., Ltd. (No. G01240136, KAIST Semiconductor Research Fund (2nd), 33\%).

\clearpage

\bibliography{references}

\clearpage
\appendix

\section{Reproducibility Statement}\label{app:reproducibility}
In an effort to ensure reproducibility, we report the description of dataset in \Cref{app:dataset_description_nlu}, \Cref{app:dataset_description_qa} and \Cref{app:dataset_description_cr}. Also we report the best hyperparameters of our experiments in \Cref{app:best_hyperparameter_nlu}, \Cref{app:best_hyperparameter_qa} and \Cref{app:best_hyperparameter_cr}. Our \textbf{SCLoRA} code to reproduce the experiment can be found at \url{https://github.com/hyowonwi/SCLoRA}.

\section{Ethical Consideration}\label{app:ethical}
We utilized publicly available datasets, including GLUE, SQuAD and commonsense reasoning, which are commonly employed in academic research, and all sources have been appropriately cited. This research does not involve any personal or confidential information, thereby eliminating concerns related to privacy. Our proposed approach and the resulting insights contribute to the advancement of artificial intelligence while adhering to principles of ethical innovation and responsibility.

\section{Broader Impact Statement}
This paper presents work whose goal is to advance the field of Machine Learning. There are many potential societal consequences of our work, none which we feel must be specifically highlighted.

\section{Use of LLMs}
In accordance with ACL 2026 policy, we acknowledge the use of LLMs in the preparation of this paper. Their use was limited solely to improving translation accuracy and ensuring grammatical correctness.

\section{Formulation \& Additional Visualization of Fisher Overlap}\label{app:fisher_overlap_vis}

\begin{figure*}[h!]
    \centering
    \small
    \begin{subfigure}[MNLI]{\includegraphics[width=0.24\textwidth]{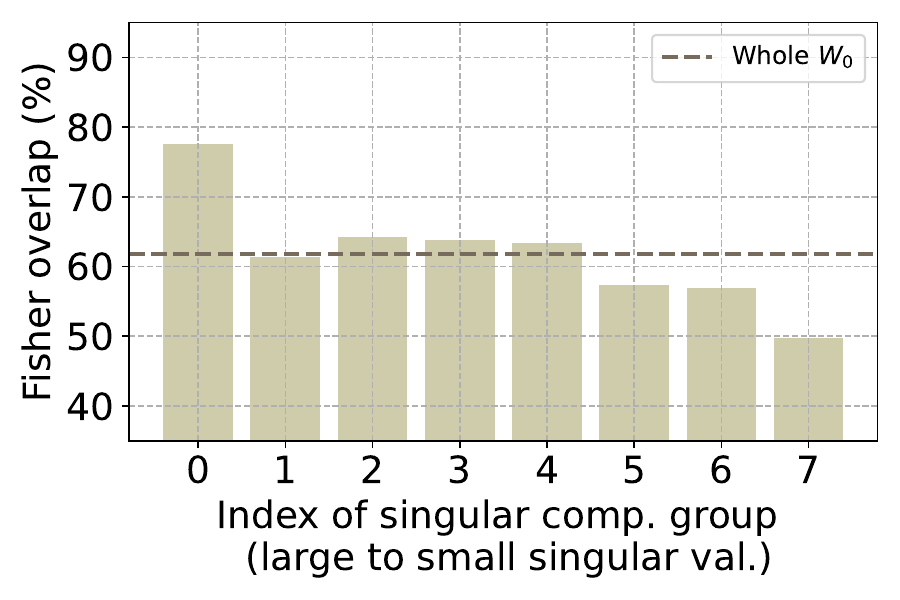}\label{fig:mnli}}
    \end{subfigure}
    \begin{subfigure}[CoLA]{\includegraphics[width=0.24\textwidth]{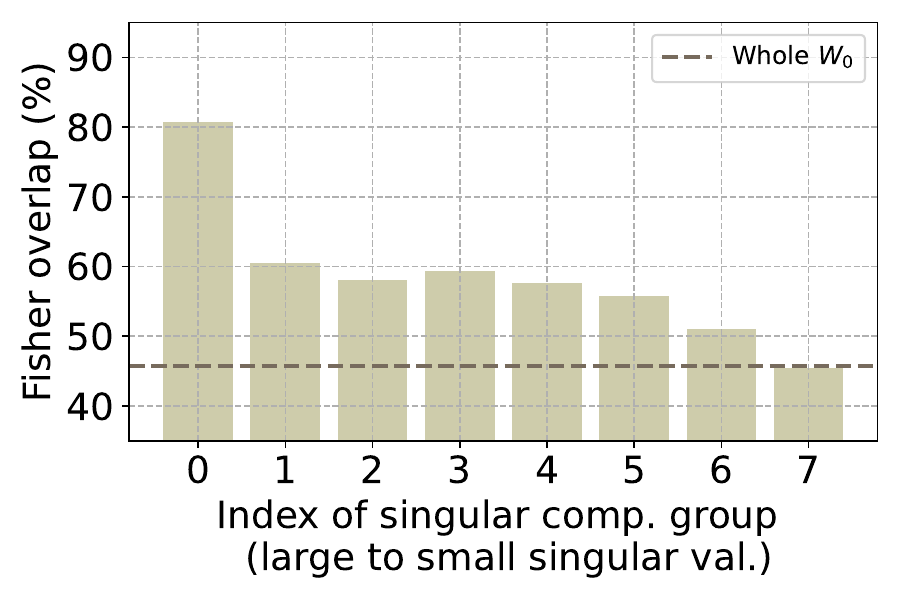}\label{fig:cola}}
    \end{subfigure}
    \begin{subfigure}[MRPC]{\includegraphics[width=0.24\textwidth]{fig/1.intro/fisher_overlap/fisher_overlap_mrpc.pdf}\label{fig:mrpc}}
    \end{subfigure}
    \begin{subfigure}[QNLI]{\includegraphics[width=0.24\textwidth]{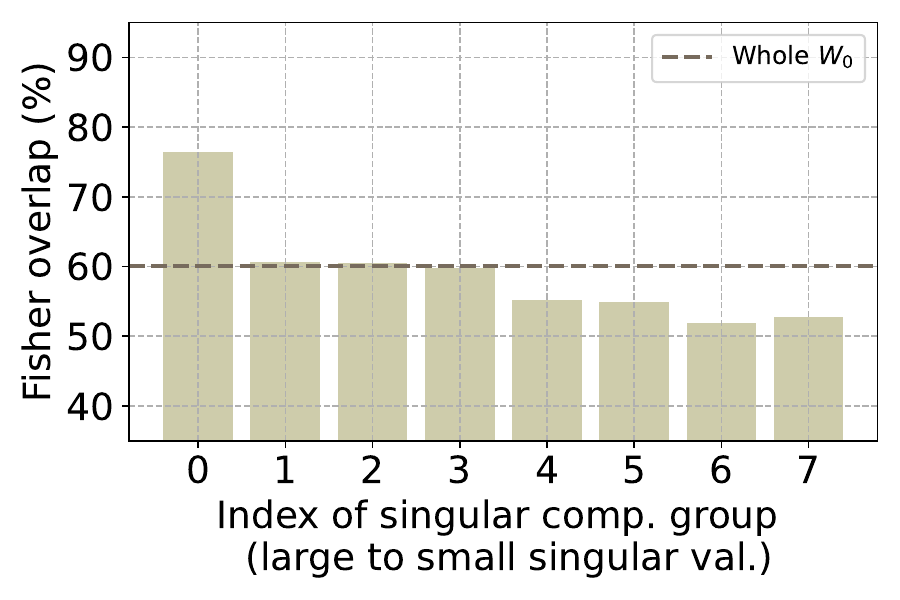}\label{fig:qnli}}
    \end{subfigure}
    \begin{subfigure}[QQP]{\includegraphics[width=0.24\textwidth]{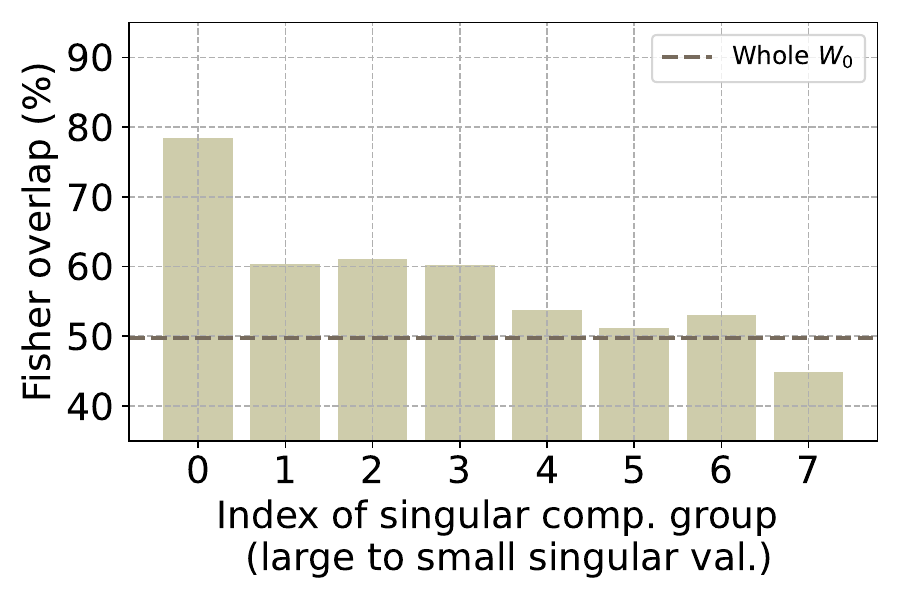}\label{fig:qqp}}
    \end{subfigure}
    \begin{subfigure}[RTE]{\includegraphics[width=0.24\textwidth]{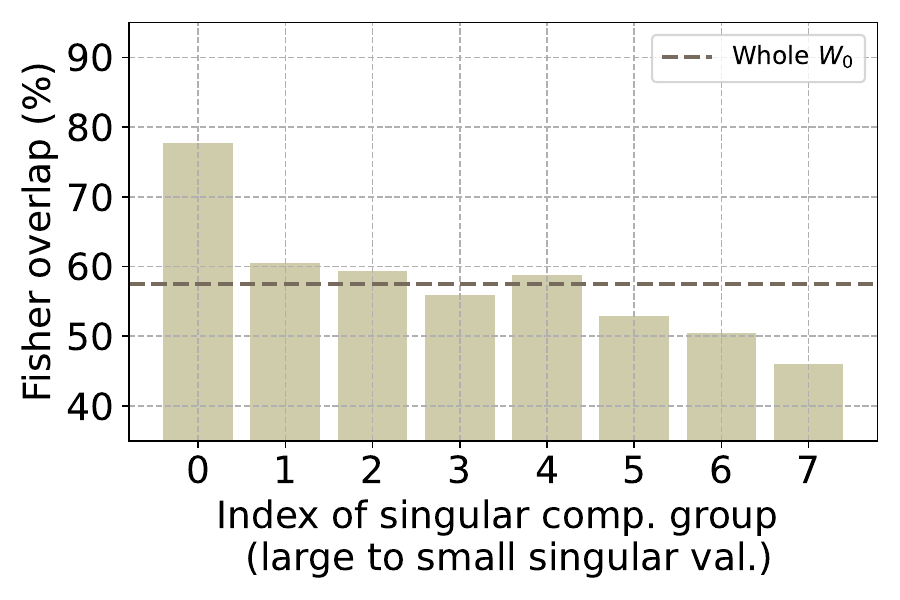}\label{fig:rte}}
    \end{subfigure}
    \begin{subfigure}[SST-2]{\includegraphics[width=0.24\textwidth]{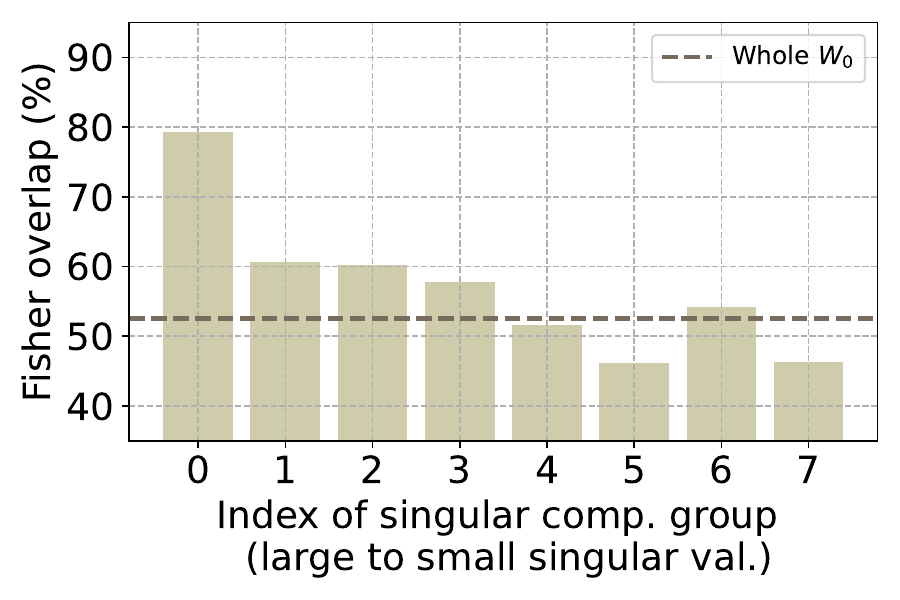}\label{fig:sst2}}
    \end{subfigure}
    \caption{The Fisher overlap~\citep{kirkpatrick2017overcoming} between the pre-trained task (BookCorpus) and the fine-tuning task (MRPC), evaluated over partial reconstructions of the pre-trained network parameters obtained by grouping singular components sorted from large to small according to their singular values.}
    \label{fig:fisher_overlap_all}
\end{figure*}

Following~\citet{kirkpatrick2017overcoming}, to examine whether different tasks solved by the same network rely on overlapping parameter subsets (see \Cref{fig:intro} (a)), we assessed the similarity of each task's Fisher information matrix. Specifically, we first computed the Fisher matrices for the two tasks, denoted by $F_{1}$ and $F_{2}$. We then normalized each matrix so that its trace was equal to 1, yielding $\hat{F}_{1}$ and $\hat{F}_{2}$. Next, we measure how closely these matrices aligned by computing the Fr\'echet distance, a metric on positive-semidefinite matrices, given as:
\begin{align}
    d^2(\hat{F}_1, \hat{F}_2)&=\frac{1}{2}\text{tr}\big(\hat{F}_1+\hat{F}_2-2(\hat{F}_1\hat{F}_2)^{1/2}\big)\\&=\frac{1}{2}\Vert \hat{F}_1^{1/2}-\hat{F}_2^{1/2}\Vert_F,
\end{align}
where this quantity lies between 0 and 1. We then define the \textit{overlap} of the two tasks' Fisher matrices as $1-d^2$. Hence, an overlap of 0 implies that the two tasks employ entirely distinct sets of parameters, whereas an overlap of 1 indicates that one Fisher matrix is simply a scaled version of the other (i.e., $F_1 = \alpha F_2$ for some $\alpha > 0$).

Then, to verify the Fisher overlap across the singular components of the pre-trained network parameters, we perform SVD on the pre-trained parameters of the RoBERTa\textsubscript{base} model, trained on multiple datasets including BookCorpus. We group the singular components sorted by their singular values and reconstruct partial versions of the parameters from each group. \Cref{fig:fisher_overlap_all} shows the Fisher overlap computed for the BookCorpus task as well as for each task in the GLUE benchmark.
Across all tasks, the Fisher overlap progressively decreases as we move from groups containing larger singular values to those with smaller ones, indicating that fine-tuning increasingly struggles to reuse the finer-grained partial reconstructions of the pre-trained parameters.

\section{Proof of \Cref{thm:map}}\label{app:proof_map}

% \begin{theorem31}
% Let $\theta_0$ and $\theta$ be the pre-trained and fine-tuned weights, respectively. Then the log probability of the prior $\log p(\theta | \mathcal{D}_A)$ for fine-tuning task can be approximated using Laplace Approximation as $\log p(\theta|\mathcal{D}_A) \simeq  f(\theta_0) - \frac{1}{2} (\theta-\theta_0)^\top F (\theta - \theta_0)$. From this, the approximated log probability of the prior is upper-bounded as follows:
% \begin{align}
%     \log \hat{p}(\theta|\mathcal{D}_A)  \leq f(\theta_0) - \lambda_{\min}(F) \sqrt{{\sum_{n=1}^{r} \sigma_n^2}},
% \end{align} where $\lambda_{\min}(\cdot)$ indicates the smallest eigenvalue and $\sigma_n$ is $n$-th singular value of $\theta-\theta_0$.
% \end{theorem31}

\begin{proof}

The optimization of neural networks can be considered as a process of estimating the network parameters $\theta$ through maximum a posteriori (MAP) estimation using the training data. This involves the pre-training dataset $\mathcal{D}_A$ and the fine-tuning dataset $\mathcal{D}_B$. The pre-trained weights are denoted as $\theta_0$, and the fine-tuned weights are represented as $\theta$.

The posterior to be maximized in the MAP estimation is formulated as:
\begin{align}
p(\theta | \mathcal{D}_A, \mathcal{D}_B)&=\frac{p(\mathcal{D}_B|\theta, \mathcal{D}_A)p(\theta|\mathcal{D}_A)}{p(\mathcal{D}_B|\mathcal{D}_A)}\notag\\&=\frac{p(\mathcal{D}_B|\theta)p(\theta|\mathcal{D}_A)}{p(\mathcal{D}_B|\mathcal{D}_A)}.  
\end{align}

Taking a logarithm of the posterior, the MAP objective becomes:
\begin{align}
    \theta^*&= \underset{\theta}{\mathrm{argmax}} \log p(\theta|\mathcal{D}_A,\mathcal{D}_B) \notag\\&= \underset{\theta}{\mathrm{argmax}} \big[ \log p(\mathcal{D}_B|\theta) + \log p(\theta|\mathcal{D}_A)\big].
\end{align}
Since the true posterior is intractable, we approximate the posterior using Laplace Approximation. $\log p(\theta|\mathcal{D}_A)$ can be expressed as a function $f(\theta)$ and approximated near the optimal point $f(\theta_0)$, where $\theta_0$ represents the pre-trained parameters, and $\nabla f(\theta_0) = 0$. Subsequently, a second-order Taylor expansion of $f(\theta)$ around $\theta_0$ is performed as follows:
\begin{align}
    \log p(\theta|\mathcal{D}_A) &\simeq f(\theta_0) + \frac{1}{2} (\theta-\theta_0) \nabla^2 f(\theta_0) (\theta - \theta_0) \notag\\&= f(\theta_0) + \frac{1}{2} (\theta-\theta_0)^\top H (\theta - \theta_0),
\end{align}where $H$ denotes the Hessian matrix of $f(\theta)$ evaluated at $\theta_0$. The expected value of the Hessian over the data distribution corresponds to the Fisher information matrix (FIM) $F$, defined as $F = -\mathbb{E}_{\mathcal{D}_A}[H]$. Following~\citep{mackay1992practical,kirkpatrick2017overcoming}, we approximate the posterior as a Gaussian distribution with mean given by the parameters $\theta_0$ and a diagonal precision given by the diagonal of the Fisher information matrix $F$. Given this approximation, the log probability can be expressed as:
\begin{align}
    \log p(\theta|\mathcal{D}_A) &\simeq \log \hat{p}(\theta|\mathcal{D}_A) \notag\\&= f(\theta_0) - \frac{1}{2} (\theta-\theta_0)^\top F (\theta - \theta_0),
\end{align} where the $F$ is symmetric and positive semi-definite, i.e., for any vector $v$, $vFv^\top \geq 0$. Then using the singular value decomposition (SVD) on $\theta - \theta_0 = \Delta \theta= U\Sigma V^\top$ where $U \in \mathbb{R}^{d_1 \times r}$,  $V \in \mathbb{R}^{r \times d_2}$, and $\Sigma \in \mathbb{R}^{r}$ with $\{\sigma_n\}_{1\leq n \leq r}$. As $U,V$ are orthonormal singular vectors and $F$ is a positive semi-definite matrix,
\begin{align}
    \Delta \theta^\top F \Delta \theta \geq \lambda_{\min}(F) \Vert \Delta \theta \Vert_F= \lambda_{\min}(F) \Vert \Sigma \Vert_F.
\end{align}
Therefore, the log probability of the approximated prior for the fine-tuning task from the pre-trained task is upper bounded as:
\begin{equation}
\begin{aligned}
    \log \hat{p}(\theta|\mathcal{D}_A) 
        &= f(\theta_0) - \tfrac{1}{2}(\theta-\theta_0)^\intercal F (\theta - \theta_0) \\&\leq f(\theta_0) - \lambda_{\min}(F)\, \Vert \Sigma \Vert_F \\
        &= f(\theta_0) - \lambda_{\min}(F)\, \sqrt{\sum_{n=1}^{r} \sigma_n^2}.
\end{aligned}
\end{equation}

\end{proof}

\section{Well-Established Properties on Laplace Approximation}\label{app:laplace_note}

\subsection{Error bound of Laplace Approximation}\label{app:higher_order_term_laplacian}

In Bayesian inference, one of the most widely used methods for approximating the posterior distribution is the Laplace approximation~\citep{kass1990validity,kirkpatrick2017overcoming,ritter2018online,wang2021afec,matena2022merging,gawlikowski2023survey}. This method expands the log-posterior function around its mode (i.e., the MAP estimate) using a Taylor series and retains terms up to the second order, thereby approximating the posterior distribution by a Gaussian distribution. In other words, higher-order terms beyond the quadratic expansion are discarded, and the resulting approximation error is generally limited and asymptotically negligible. 

In particular, under standard regularity conditions, it is well established that the relative error of Laplace approximation is no worse than $\mathcal{O}_p(n^{-1})$ under standard regularity conditions, where $\mathcal{O}_p$ refers to stochastic boundedness~\citep{kass1990validity,bilodeau2023tightness}. This ensures that, as the sample size $n$ increases, the approximation error vanishes at the rate of $n^{-1}$. 
Moreover, in deep learning, where the number of training samples $n$ is typically very large, the accuracy of the Laplace approximation is further reinforced. Consequently, the Laplace approximation provides not only a practical tool but also a theoretically justified method for posterior approximation in both Bayesian inference and large-scale probabilistic modeling.

\subsection{Decay Rate of the Integral over the Mode-Distant Region}\label{app:mode_distant_term_laplacian}

% Laplace approximation is applied when the target function is sharply concentrated around a single mode $\theta_{0}$ and decays rapidly as $\theta$ moves away from it. The integral is rewritten in exponential form, and the log-posterior is approximated by a second-order Taylor expansion around $\theta_{0}$. In this setting, regions where $\lVert \theta - \theta_{0} \rVert$ is large correspond to the tail of the function. Importantly, the approximation error in this tail region should not be evaluated solely by the magnitude of $\lVert \theta - \theta_{0} \rVert$, but by its \emph{actual contribution to the integral}. Building on the classical analysis of \citet{kass1990validity}, the posterior mass increasingly concentrates around $\theta_{0}$ as the sample size $n$ grows, while the integral over mode-distant regions becomes exponentially small. Formally, for a sufficiently small neighborhood $B_{\delta}(\theta_{0})$ of the mode, the contribution of the mode-distant region satisfies 
% \begin{align} \int_{\Theta-B_{\delta}(\theta_{0})} g(\theta)\exp\!\bigl\{l_{n}(\theta)-l_{n}(\theta_{0})\bigr\}\, d\theta \le \exp(-nc), 
% \end{align}
% for some constants $c>0$ and $\alpha>0$ That is, the integral over regions distant from $\theta_{0}$ decays exponentially in $n$, implying that its contribution to the overall integral vanishes asymptotically. Consequently, the principal contribution to the Laplace approximation arises from the local quadratic neighborhood around $\theta_{0}$, and the approximation error is dominated by this region rather than by the tail.

Laplace approximation is applied when the target function is sharply concentrated around a single mode $\theta_0$ and decays rapidly as $\theta$ moves away from it. The method rewrites the integral in exponential form and then approximates the log-posterior by a second-order Taylor expansion around its mode. The region where $\theta-\theta_0$ is large corresponds to the tail of the function. The approximation error in this region should not be judged solely by the magnitude of $\lvert \theta - \theta_0 \rvert$; rather, its actual contribution to the integral must be considered.
\citet{kass1990validity} rigorously demonstrates that this tail contribution is negligible. First, as the sample size increases, the likelihood function becomes increasingly peaked, so the posterior concentrates around the mode $\theta_0$. Second, the integral over regions distant from $\theta_0$ decays exponentially with $n$, that is, it is bounded by $\exp(-nc)$ for $c>0$. Consequently, the integral over the mode-distant region converges to zero, and the principal contribution to the integral arises near the mode.

\section{Exponential Decay of Singular Values}\label{app:singular_value_decay}

\begin{figure}[h]
    \centering
    \includegraphics[width=0.9\columnwidth]{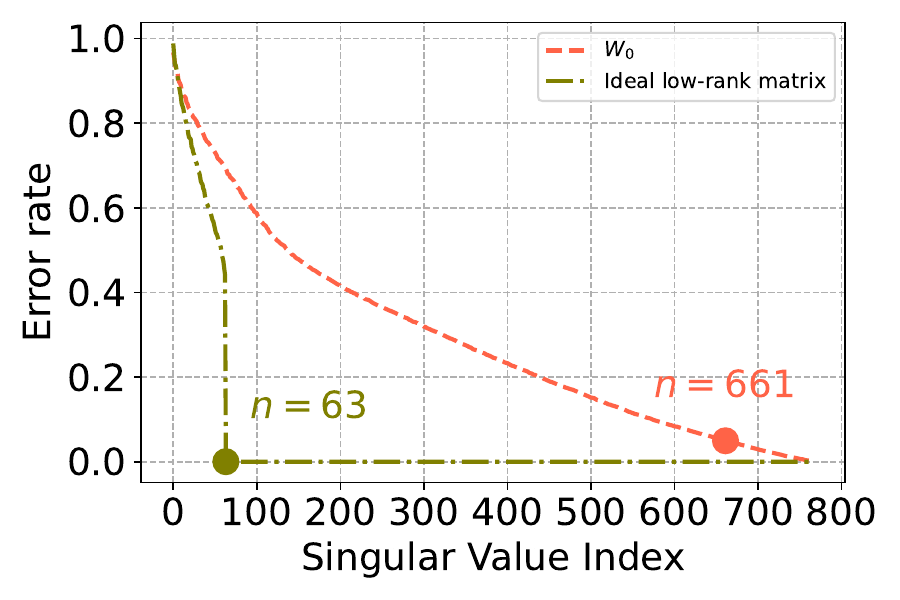}
    \caption{The error rate of the normalized singular values for: (i) the final output projection layer weights $W_0$ in the self-attention mechanism of DeBERTaV3\textsubscript{base}, and (ii) an ideal low-rank matrix with rank $r=64$. The marker indicates the $n$-value where the approximation error reaches 5\%.}
    \label{fig:error_rate}
\end{figure}

To find the best possible $n$-dimensional subspace $V_n$ such that the closest approximation $v \in V_n$ to $W$ minimizes the error $\Vert W - v\Vert_X$, the definition of Kolmogorov $n$-width is formulated as follows:
\begin{align}
d_n(W, X) = \inf_{\substack{V_n \subset X \\ \dim V_n = n}}  \inf_{v \in V_n} \Vert W - v \Vert_X,
\end{align} where $V_n$ is $n$-dimensional subspace of $X$, $v$ is an element from the subspace $V_n$. `inf' stands for infimum. When using the Frobenius norm (or spectral norm) with matrices, the Kolmogorov $n$-width is computed by the singular values of $W$ as follows:
\begin{align}
    d_n(W,X)=\sigma_{n+1},
\end{align} where $\sigma_{n+1}$ is the $(n+1)$-th largest singular value of the matrix $W$. The Kolmogorov $n$-width measures how well a set $W$ can be approximated by an $n$-dimensional subspace. In other words, it represents the minimal maximum error when approximating with an $n$-dimensional subspace. Then we can determine the optimal dimensionality needed to achieve a desired approximation accuracy.

If the singular values decrease rapidly, $W$ can be well approximated even for small $n$, and the Kolmogorov $n$-width also decreases quickly. Therefore, the singular value decay rate $\alpha$, which plays a pivotal role in determining how effectively a matrix can be approximated, is commonly modeled by an exponential decay function as follows: 
\begin{align} 
\sigma'_n = C e^{-\alpha n},
\end{align} where $\sigma'_n$ represents the $n$-th modeled singular values, $C>0$ is a constant, and $\alpha>0$ is the decay rate. When the decay rate $\alpha$ is low, the singular values decrease gradually, resulting in large errors when approximating with the same $n$ dimensions. To minimize the approximation errors, a larger $n$ is required, indicating that significant information is contained in the lower singular values.

\paragraph{Empirical analysis of the Kolmogorov $n$-width.}
To empirically analyze the Kolmogorov $n$-width of the pre-trained language model, we present error rates based on low-rank approximation under the same conditions as shown in \Cref{fig:intro} (b) of Introduction. The formulation of error rates $E_W(n)$ is as follows: 
\begin{align}
E_W(n) = \left( \frac{\| W - v \|_F}{\| W \|_F} \right) \times 100 \%,
\end{align} where $W$ is the original matrix and $v$ is the approximated matrix obtained by truncating the SVD to rank $n$. The error rates for the pre-trained model and the ideal low-rank matrix are presented in \Cref{fig:error_rate}, with markers indicating the $n$-value where the error rate reaches 5\%. For the ideal low-rank matrix, the rank at which the error rate reaches 5\% is 63. This suggests that the matrix has a low-dimensional structure, with the most important information concentrated in the top singular values. The lower singular values have little effect on the approximation and can be considered noise. In contrast, for the pre-trained model, the $n$-value required to reach 95\% approximation is 661, which is significantly larger than ideal low rank matrix. This indicates that the data is complex and high-dimensional, and the lower singular values contain important information rather than merely noise.

\section{Discussion on Constraining the Frobenius Norm of $\Delta W$}\label{app:discuss_frobenius}

\Cref{thm:map} provides an upper bound of approximated posterior based on the Frobenius norm of $\Delta \theta$ and the minimum eigenvalue of the Fisher Information Matrix (FIM), $\lambda_{\min}(F)$. However, simply constraining the Frobenius norm is not sufficient. This is because the FIM $F$ encodes information about \textit{important directions} about the pre-trained task in the parameter space. 
Specifically, \Cref{eq:approx_posterior} provides the approximated log posterior of pre-trained task using Laplace approximation.
The FIM $F$ is a positive semi-definite matrix. The quadratic term quantifies how much the parameter shift $\Delta \theta = \theta - \theta_0$ affects the output distribution of the original task. In low-rank update methods like LoRA, the parameter update can be expressed as $\Delta \theta = U \Sigma V^\top$. Due to the low-rank nature of the update, the singular values $\sigma_i$ often become concentrated in a small number of singular directions.

Even for the same Frobenius norm of $\Delta \theta$, if $\sigma_n$ is heavily concentrated in certain directions, and those directions align with sensitive ones in the FIM (i.e., those with large eigenvalues), the penalty term $\Delta \theta^\top F \Delta \theta$ can become significantly larger (not $\Delta \theta^\top \Delta \theta$). As a result, the approximated posterior decreases more sharply, which intensifies catastrophic forgetting on the pre-trained task. Therefore, the Frobenius norm merely controls the total magnitude of change, but does not prevent the change from being concentrated in directions that are crucial for the original task. In contrast, \textbf{SCLoRA} explicitly controls individual singular values via clipping, thereby directly limiting updates in directions where the model is most sensitive. This goes beyond a simple norm constraint and introduces a structural bias that suppresses updates in forgetting-prone directions. Therefore, while the Frobenius norm constrains only the total size of the update, \textbf{SCLoRA} enables spectral control via clipping over which components grow, effectively mitigating catastrophic forgetting in a more targeted and principled way. This is also well-described in~\citet{kirkpatrick2017overcoming}, where L2-norm cannot mitigate catastrophic forgetting because important parameters from previous tasks are not adequately preserved. The standard L2-norm regularization treats all parameters uniformly, failing to account for their varying importance, and thus cannot effectively mitigate this problem.

\section{Algorithm of \textbf{SCLoRA}}\label{app:train_algorithm}
In this section, we summarize the detailed algorithm of \textbf{SCLoRA} in \Cref{alg:train}.

\begin{algorithm}[h]
\KwIn{Dataset $\mathcal{D}$; total iterations $T$; learning rate $\eta$; $\gamma$, $\bar{\sigma}_k.$}
\For{$t = 1, \dots, T$}{

    % $\Sigma^{(t)}_k = \text{clamp}(\Sigma_k^{(t)}, 0, \bar{\sigma})$\;
    $\Sigma^{(t)}_k = \text{min}(\text{max}(\Sigma^{(t)}_k,0),\bar{\sigma}_k)$\;

    $W^{(t)}_k=W_0+ U^{(t)}_k \Sigma^{(t)}_k (V^{(t)}_k)^\top$\;
     
    $U_k^{(t+1)} = U_k^{(t)} - \eta \nabla_{U_k}(\mathcal{L}(U_k^{(t)}, \Sigma_k^{(t)}, V_k^{(t)}) + \gamma R(U_k^{(t)}, V_k^{(t)}))\;$
    
    $V_k^{(t+1)} = V_k^{(t)} - \eta \nabla_{V_k}(\mathcal{L}(U_k^{(t)}, \Sigma_k^{(t)}, V_k^{(t)}) + \gamma R(U_k^{(t)}, V_k^{(t)}))\;$
    
    $\Sigma_k^{(t+1)} = \Sigma_k^{(t)} - \eta \nabla_{\Sigma_k} \mathcal{L}(U_k^{(t)}, \Sigma_k^{(t)}, V_k^{(t)})$\;
}
\KwOut{The fine-tuned parameters $\{U^{(T)}, \Sigma^{(T)}, V^{(T)}\}$; $W^{(T)} = W_0 + U^{(T)} \Sigma^{(T)} (V^{(T)})^\top$.}
\caption{How to train \textbf{SCLoRA}}\label{alg:train}

\end{algorithm}

\section{LoRA with Discrete Fourier Transform}\label{app:dft}

\begin{table*}[h]
    \centering
    \label{app:tab_nlu_fourierft_config}
    \resizebox{0.85\textwidth}{!}{%
    \begin{tabular}{l | c | ccccccc }
        \toprule
        Method & \# Params & SST-2 & MRPC & CoLA & QNLI & RTE & STS-B & Avg. \\
        \midrule
        FT & 125M & 94.8\textsubscript{$\pm$0.2} & 90.2\textsubscript{$\pm$0.6} & 63.6\textsubscript{$\pm$2.6} & 92.8\textsubscript{$\pm$0.2} & 78.7\textsubscript{$\pm$0.7} & 91.2\textsubscript{$\pm$0.5} & 85.2 \\
        \midrule
        BitFit & 0.1M & 93.7\textsubscript{$\pm$0.3} & \textbf{92.7\textsubscript{$\pm$0.6}} & 62.0\textsubscript{$\pm$1.9} & 91.8\textsubscript{$\pm$0.2} & \underline{81.5}\textsubscript{$\pm$0.8} & 90.8\textsubscript{$\pm$0.6} & \underline{85.4} \\
        Adpt$^D$ & 0.3M & 94.2\textsubscript{$\pm$0.1} & 88.5\textsubscript{$\pm$1.1} & 60.8\textsubscript{$\pm$1.1} & 93.1\textsubscript{$\pm$0.4} & 71.5\textsubscript{$\pm$2.7} & 89.7\textsubscript{$\pm$0.7} & 83.0 \\
        Adpt$^D$ & 0.9M & 94.7\textsubscript{$\pm$0.3} & 88.4\textsubscript{$\pm$0.1} & 62.6\textsubscript{$\pm$0.9} & 93.0\textsubscript{$\pm$0.2} & 75.9\textsubscript{$\pm$0.4} & 90.3\textsubscript{$\pm$0.2} & 84.2 \\
        LoRA & 0.3M & \textbf{95.1\textsubscript{$\pm$0.2}} & 89.7\textsubscript{$\pm$0.7} & 63.4\textsubscript{$\pm$1.0} & \underline{93.3}\textsubscript{$\pm$0.2} & 78.4\textsubscript{$\pm$0.2} & \textbf{91.5\textsubscript{$\pm$0.2}} & 85.2 \\
        AdaLoRA & 0.3M & 94.5\textsubscript{$\pm$0.2} & 88.7\textsubscript{$\pm$0.5} & 62.0\textsubscript{$\pm$0.5} & 93.1\textsubscript{$\pm$0.3} & 81.0\textsubscript{$\pm$0.9} & 90.5\textsubscript{$\pm$0.2} & 85.0 \\
        DyLoRA & 0.3M & 94.3\textsubscript{$\pm$0.2} & 89.5\textsubscript{$\pm$0.6} & 61.1\textsubscript{$\pm$0.9} & 92.2\textsubscript{$\pm$0.4} & 78.7\textsubscript{$\pm$1.0} & 91.1\textsubscript{$\pm$0.3} & 84.5 \\
        FourierFT & 0.024M & 94.2\textsubscript{$\pm$0.3} & 90.0\textsubscript{$\pm$0.3} & \underline{63.8}\textsubscript{$\pm$1.6} & 92.2\textsubscript{$\pm$0.1} & 79.1\textsubscript{$\pm$0.2} & 90.8\textsubscript{$\pm$0.4} & 85.0 \\
        \textbf{SCLoRA} & 0.3M & \underline{95.0}\textsubscript{$\pm$0.2} & \underline{90.6}\textsubscript{$\pm$1.1} & \textbf{65.1}\textsubscript{$\pm$0.1} & \textbf{93.4}\textsubscript{$\pm$0.2} & \textbf{83.0}\textsubscript{$\pm$1.0} & \underline{91.3}\textsubscript{$\pm$0.1} & \textbf{86.4} \\
        \bottomrule
    \end{tabular}
    }
    \caption{Comparison of various methods with RoBERTa\textsubscript{base} on GLUE tasks with the experimental setup from \citet{gao2024fourierft}. The results for the baselines are copied from~\citet{gao2024fourierft}. The best performance is set in \textbf{bold}, and the second best is set in \underline{underline}.}
    
\end{table*}

Discrete Fourier Transform (DFT)~\citep{briggs1995dft} converts a finite sequence of equally-spaced samples of a function into a same-length sequence of equally-spaced samples of the discrete-time Fourier transform (DTFT), which is a complex-valued function of frequency. Recent studies have proposed the DFT-based parameter efficient fine-tuning method. FouRA~\citep{borse2024foura} transforms the hidden vectors in the latent space into the singular domain using DFT and compute LoRA, followed by reconstruction through inverse DFT. 
FourierFT~\citep{gao2024fourierft} considers the adapter as a 2D spatial-domain matrix and transforms into a 2D DFT spectrum. 
Therefore, existing methods have focused on DFT-based approaches to either flexibly select adapter ranks depending on the input or reduce the number of parameters.
On the other hand, \textbf{SCLoRA} leverages a parameterized SVD in the parameter space, and injects the singular components with spectral clipping to achieve effective adaptation and mitigation of catastrophic forgetting.

We conduct additional experiments following experimental setup of FourierFT~\citep{gao2024fourierft}, one of the LoRA variants that employs the concept of DFT. Strictly following~\citep{gao2024fourierft}, we fine-tune RoBERTa\textsubscript{base}, adapting only the query and value projection matrices using LoRA. We maintain the same experimental settings as the original work, while only searching for the learning rate from $\{\num{1e-3}, \num{8e-4}, \num{6e-4}, \}$, $\gamma$ from $\{\num{1e-2}, \num{3e-3}, \num{1e-3}\}$, and $q$ from $\{0.25, 0.5, 0.75, 1.0\}$ and the results are reported in~\Cref{app:tab_nlu_fourierft_config}.

\section{How does the Adapters $\Delta W$ Compared to $W$?}

\begin{table*}[!t]
\label{tab:frob_norm}
\centering
\small
\setlength{\tabcolsep}{3pt}
\begin{tabular}{c | ccc | c | c | cc} \toprule
\multirow{2}{*}{Model} & \multicolumn{3}{c|}{$\Vert U^\top W V \Vert_F$} & \multirow{2}{*}{$\Vert U_{W_q}^\top \Delta W V_{W_q} \Vert_F$}  & \multirow{2}{*}{$\Vert \Delta W \Vert_F$} &  \multirow{2}{*}{\texttt{Factor}\textsubscript{$W \rightarrow \Delta W$}} & \multirow{2}{*}{\texttt{Factor}\textsubscript{$\Delta W\rightarrow W$}}  \\ \cmidrule{2-4}
                      &  $\Delta W_q$  & $W_q$ & Random & & & &  \\ \midrule
          LoRA        & 0.48 & 11.22 & 0.32 & 0.16 & 3.81 & 7.94 & 23.82 \\
          PiSSA       & 0.38 & 11.22 & 0.35 & 0.11 & 2.49 & 6.54 & 22.60 \\
          MiLoRA      & 0.45 & 11.22 & 0.35 & 0.08 & 3.11 & 6.91 & 38.86 \\
          \textbf{SCLoRA}      & 0.36 & 11.22 & 0.38 & 0.03 & 0.94 & 2.60 & 31.18 \\ \midrule
\end{tabular}
\caption{The Frobenius norm of $U^\top WV$, where $U$ and $V$ are the left and right top $r$ singular vector directions of either: (1) $\Delta W_q$, (2) $W_q$, or (3) a random matrix. (4) The Frobenius norm of $U^\top \Delta W V$, where $U$ and $V$ are from $W_q$. (5) The Frobenius norm of $\Delta W$. (6,7) The introduced factors. The weights are taken from the last query layer of RoBERTa\textsubscript{base}, fine-tuned on STS-B dataset with $r=8$.}
\end{table*}

We explore the relationship between $\Delta W$ and $W$ by measuring the correlation between $\Delta W$ and $W$ as well as the magnitude of $\Delta W$ in comparison to its corresponding directions in $W$. To do so, we introduce two key factors:

\begin{itemize}
    \item \texttt{Factor}\textsubscript{$W\rightarrow\Delta W$} is a factor formulated as ${\Vert \Delta W\Vert_F}/{\Vert U_{\Delta W}^\top W V_{\Delta W} \Vert_F}$, which indicates the ratio of the norm of difference over the norm of projected $W$ on the $r$-dimensional subspace of $\Delta W$. This factor is also called \textit{amplification factor}~\citep{hu2021lora}, measuring how the new information of $\Delta W$ is related to the existing information of $W$. A larger ratio refers that the task-specific information of $W$ has been amplified in $\Delta W$.
    \item \texttt{Factor}\textsubscript{$\Delta W\rightarrow W$} is a factor formulated as ${\Vert \Delta W\Vert_F}/{\Vert U_{W}^\top \Delta W V_{W} \Vert_F}$, which is the ratio of the norm of difference over the norm of projected $\Delta W$ on the $r$-dimensional subspace of $W$. It indicates the extent to which the change aligns with $W$. A larger ratio refers that $\Delta W$ has learned new information that is not present in $W$.
\end{itemize}

Following~\citet{hu2021lora}, we project $W$ onto the $r$-dimensional subspace of $\Delta W$ by computing $U^\top W V$, where $U$ and $V$ are the left and right singular vectors of $\Delta W$, $W$, and the random matrix. Additionally, we project $\Delta W$ onto the subspace of $W$ by computing $U^\top \Delta W V$. As shown in \Cref{tab:frob_norm}, \textbf{SCLoRA} and other methods exhibit similar Frobenius norms when $W$ is projected onto the subspace of $\Delta W$, $W$ and random matrix. However, compared to the baselines, the projection of $\Delta W$ onto the subspace of $W$ in \textbf{SCLoRA} shows the lowest correlation with a value of 0.02, which is less than half of the smallest baseline. This suggests that \textbf{SCLoRA} processes the existing information in $W$ similarly to other methods, while being better at learning independent new information without relying on the existing information in $W$. Furthermore, considering the Frobenius norm of $\Delta W$, both LoRA and PiSSA exhibit a large $\texttt{Factor}_{W \rightarrow \Delta W}$ and a small $\texttt{Factor}_{\Delta W \rightarrow W}$, indicating that $\Delta W$ primarily amplifies information already present in $W$. MiLoRA also shows a large $\texttt{Factor}_{\Delta W \rightarrow W}$, but this results from the large magnitude of $\Delta W$, leading to significant changes from the pre-trained weights. In contrast, \textbf{SCLoRA} exhibits a relatively small $\texttt{Factor}_{W \rightarrow \Delta W}$ of 4.25 but a large $\texttt{Factor}_{\Delta W \rightarrow W}$ of 46.77. Given the small magnitude of $\Delta W$, this indicates that \textbf{SCLoRA} stands out for its ability to learn new information that is not already in $W$ with minimal deviation from the pre-trained weights. 
We define the following ratio based on these two factors:
\begin{align}
\texttt{Ratio} = \frac{\texttt{Factor}_{\Delta W \rightarrow W}}{\texttt{Factor}_{W \rightarrow \Delta W}}.
\end{align}
A higher Ratio indicates that $\Delta W$ contains new components that do not lie within the subspace of the pre-trained weights.
This value is significantly higher than those of existing LoRA-based methods, demonstrating that \textbf{SCLoRA} is not merely amplifying existing directions, but instead learning new task-specific directions that do not exist in the pre-trained weights.
According to~\Cref{tab:frob_norm}, \textbf{SCLoRA} achieves a \texttt{Ratio} of $31.18/2.60\simeq11.99$, which is significantly higher than other LoRA-based methods. This provides strong evidence that \textbf{SCLoRA} is not merely amplifying existing directions, but rather learning new information that is not already in $W$ from the pre-trained weights.
In summary, \textbf{SCLoRA} preserves pre-trained knowledge while injecting fine-grained information into low-alignment regions, thereby naturally inducing a divergence of subspaces during fine-tuning.

\section{Experimental Setup}
\subsection{Natural Language Understanding}
\subsubsection{Dataset description}\label{app:dataset_description_nlu}
We describe the benchmark datasets of GLUE~\citep{wang2018glue} below~\footnote{\url{https://huggingface.co/datasets/nyu-mll/glue}}.
\begin{itemize}
    \item \textbf{CoLA.} The Corpus of Linguistic Acceptability~\citep{warstadt2019cola} provides a dataset of English sentences, where each sentence is judged for grammatical acceptability based on data from books and journal articles. The objective is a binary classification to determine whether a sentence is grammatically correct or incorrect. The dataset consists of 8.5k samples for training, 1k samples for validation, and 1k samples for test.
    \item \textbf{SST-2.} The Stanford Sentiment Treebank~\citep{socher2013sst2} includes sentences from movie reviews, along with human-provided sentiment annotations. The goal is to classify the sentiment of each sentence as either positive or negative. The dataset consists of 67k samples for training, 872 samples for validation, and 1.8k samples for test.
    \item \textbf{MRPC.} The Microsoft Research Paraphrase Corpus~\citep{dolan2005mrpc} contains pairs of sentences automatically extracted from online news sources. Human annotators label each pair, and the task is to identify whether the two sentences in a pair convey the same meaning. The dataset consists of 3.7k samples for training, 408 samples for validation, and 1.7k samples for test.
    \item \textbf{QQP.} The Quora Question Pairs dataset~\citep{chen2018qqp} consists of question pairs taken from Quora, a community-driven question-and-answer platform. The task is to determine if two given questions are semantically identical. The dataset consists of 364k samples for training, 40k samples for validation, and 391k samples for test.
    \item \textbf{MNLI.} The Multi-Genre Natural Language Inference Corpus~\citep{williams2017mnli} includes sentence pairs with textual entailment annotations collected through crowdsourcing. Given a premise and a hypothesis, the task is to predict whether the premise entails the hypothesis, contradicts it, or is neutral. The dataset includes both in-domain and cross-domain evaluations using a hidden test set. The dataset consists of 393k samples for training, 20k samples for validation, and 20k samples for test.
    \item \textbf{QNLI.} The Question-Answering Natural Language Inference dataset~\citep{wang2018qnli} consists of question-paragraph pairs from which an answer must be found. The task involves determining whether a specific sentence from the paragraph answers the corresponding question. The dataset consists of 108k samples for training, 5.7k samples for validation, and 5.7k samples for test.
    \item \textbf{RTE.} The Recognizing Textual Entailment dataset~\citep{bentivogli2009rte} comes from a series of annual challenges focusing on textual entailment. The task is to classify sentence pairs as either entailment or non-entailment. The dataset consists of 2.5k samples for training, 276 samples for validation, and 3k samples for test.
    \item \textbf{STS-B.} The Semantic Textual Similarity Benchmark~\citep{cer2017stsb} features sentence pairs drawn from various sources, including news headlines and image captions, with human-assigned similarity scores. The task is a regression problem where the model must predict a similarity score ranging from 0 to 5. The dataset consists of 7k samples for training, 1.5k samples for validation, and 1.4k samples for test.
\end{itemize}

\subsubsection{Experimental setup}\label{app:setup_nlu}
We evaluate \textbf{SCLoRA} on the General Language Understanding Evaluation (GLUE) benchmark~\citep{wang2018glue}, which includes 3 categories of natural language understanding tasks: i) single-sentence (CoLA and SST-2); ii) similarity and paraphrasing (MRPC, QQP, and STS-B); iii) natural language inference tasks (MNLI, QNLI, and RTE). For a fair comparison, following \citet{hu2021lora}, we adopt the pre-trained RoBERTa\textsubscript{base} as the backbone model. We use  1 GPU of NVIDIA RTX A6000 for experiments. We report Matthews correlation for CoLA, Spearman correlations for STS-B, and accuracy scores for the other tasks. We conduct the experiments in Huggingface Framework~\footnote{\url{https://github.com/huggingface/transformers}}.

\subsubsection{Hyperparameters}\label{app:best_hyperparameter_nlu}

\begin{table*}[h]
\label{tab:parmeter_nlu}
\centering
\small
\setlength{\tabcolsep}{2.5pt}
\begin{tabular}{l ccccccc} \toprule
Dataset & Learning rate & Batch size & \#Epochs & Metric &  $q$ & $\gamma$ & How to initialize $U,V$ \\ \midrule

CoLA & \num{4e-4} & 32 & 25 & Matthews correlation & 0.5 & \num{3e-2} & random $r$ singular vectors \\
MNLI & \num{5e-4} & 32 & 7 & Accuracy & 0.5 & \num{1e-1} & 0, random Gaussian \\
MRPC & \num{4e-4} & 16 & 30 & Accuracy & 0.5 & \num{1e-2} & random $r$ singular vectors  \\
QNLI & \num{4e-4} & 32 & 5 & Accuracy & 0.25 & \num{1e-1} & 0, random Gaussian \\
QQP & \num{5e-4} & 32 & 5 & Accuracy & 0.25 & \num{1e-3} & 0, random Gaussian  \\
RTE & \num{5e-4} & 32 & 50 & Accuracy & 1.0 & \num{5e-2} & 0, random Gaussian \\
SST-2 & \num{5e-4} & 32 & 24 & Accuracy & 1.0 & \num{1e-1} & 0, random Gaussian \\ 
STS-B & \num{4e-4} & 32 & 25 & Pearson correlation & 0.25 & \num{1e-1} & 0, random Gaussian \\ \bottomrule

\end{tabular}
\caption{Best hyperparameters for \textbf{SCLoRA} in natural language understanding for RoBERTa\textsubscript{base}}

\end{table*}

\begin{table*}[h]
\label{tab:parmeter_deberta_r8}
\centering
\small
\setlength{\tabcolsep}{5pt}
\begin{tabular}{l cccccc} \toprule
Dataset & Learning rate & Batch size & \#Epochs & Metric & $q$ & $\gamma$ \\ \midrule

CoLA & \num{8e-4} & 32 & 25 & Matthews correlation & 1.0 & \num{5e-1}\\
MNLI & \num{5e-4} & 32 & 7 & Accuracy & 0.75 & \num{1e-2}  \\
MRPC & \num{1e-3} & 32 & 30 & Accuracy & 0.75 & \num{5e-1}  \\
QNLI & \num{5e-4} & 32 & 5 & Accuracy & 0.75 & \num{1e-1}  \\
QQP & \num{8e-4} & 32 & 5 & Accuracy & 0.25 & \num{1e-2}  \\
RTE & \num{1.2e-3} & 32 & 50 & Accuracy & 0.75 & \num{1e-1}  \\
SST-2 & \num{8e-4} & 32 & 24 & Accuracy & 1.0 & \num{1e-1}  \\ 
STS-B & \num{2.2e-3} & 32 & 25 & Pearson correlation & 0.5 & \num{5e-1}  \\ \bottomrule

\end{tabular}
\caption{Best hyperparameters for \textbf{SCLoRA} in natural language understanding for DeBERTaV3\textsubscript{base} with $r=8$}

\end{table*}

\begin{table*}[h]
\label{tab:parmeter_deberta_r2}
\centering
\small
\setlength{\tabcolsep}{2.5pt}
\begin{tabular}{l cccccc} \toprule
Dataset & Learning rate & Batch size & \#Epochs & Metric & $q$ & $\gamma$ \\ \midrule

CoLA & \num{8e-4} & 32 & 25 & Matthews correlation & 1.0 & \num{1.1e-1}\\
MNLI & \num{5e-4} & 32 & 7 & Accuracy & 0.75 & \num{1e-1}  \\
MRPC & \num{1e-3} & 32 & 30 & Accuracy & 0.5 & \num{1e-2}  \\
QNLI & \num{7e-4} & 32 & 5 & Accuracy & 0.75 & \num{1e-1}  \\
QQP & \num{8e-4} & 32 & 5 & Accuracy & 0.25 & \num{5e-3}  \\
RTE & \num{1.2e-3} & 32 & 50 & Accuracy & 0.75 & \num{1e-1}  \\
SST-2 & \num{8e-4} & 32 & 24 & Accuracy & 1.0 & \num{5e-1}  \\ 
STS-B & \num{2.2e-3} & 32 & 25 & Pearson correlation & 1.0 & \num{6e-1}  \\ \bottomrule

\end{tabular}
\caption{Best hyperparameters for \textbf{SCLoRA} in natural language understanding for DeBERTaV3\textsubscript{base} with $r=2$}

\end{table*}

To tune \textbf{SCLoRA} with RoBERTa\textsubscript{base}, we search for the learning rate from $\{\num{4e-4}, \num{5e-4}\}$, $q$ from $\{0.25, 0.5, 0.75, 1.0\}$ and $\gamma$ from $\{\num{1e-1}, \num{7e-2}, \num{5e-2}, \num{3e-2}, \num{1e-2}, \num{1e-3}\}$. The learnable singular vectors $U/V$ can be initialized as i) random $r$ singular vectors of $W_0$, ii) $U$ with zeros, $V$ with random Gaussian initialization. To tune \textbf{SCLoRA} with DeBERTaV3\textsubscript{base}, we search $q$ from $\{0.25, 0.5, 0.75, 1.0\}$ and $\gamma$ from $\{\num{1e-1}, \num{1.1e-1}, \num{1e-2}, \num{5e-1}, \num{6e-1}\}$. The learnable singular vectors $U/V$ are be initialized as $U$ with zeros and $V$ with random Gaussian initialization. We report the best hyperparameters of \textbf{SCLoRA} in \Cref{tab:parmeter_nlu} to \Cref{tab:parmeter_deberta_r2}.

\subsubsection{Experimental Result with standard deviations}\label{app:main_nlu_std}
We report the experimental results on GLUE tasks with standard deviation of Roberta\textsubscript{base} and DeBERTa\textsubscript{base}V3 in \Cref{tab:main_nlu_std} and \Cref{tab:main_nlu_deberta_std}, respectively. Note that the results for \Cref{tab:main_nlu_deberta_std} are copied from~\citet{zhang2023adalora}, the results with standard deviation are reported only for \textbf{SCLoRA}.

\begin{table*}[h]
\label{tab:main_nlu_std}
\centering
\small
\setlength{\tabcolsep}{1.3pt}
\begin{tabular}{l | ccccccccc} \toprule
% Method & MNLI  & SST-2  & CoLA  & QQP  & QNLI  & RTE  & MRPC  & STS-B & Avg. \\ \midrule
        Method  &  MNLI & SST-2 & CoLA & QQP & QNLI & RTE & MRPC & STS-B & Avg. \\ \midrule
        LoRA  & 87.93\textsubscript{$\pm$0.15} & 94.80\textsubscript{$\pm$0.11} & 64.49\textsubscript{$\pm$0.64} & 90.94\textsubscript{$\pm$0.04} & 92.73\textsubscript{$\pm$0.11} & 80.39\textsubscript{$\pm$0.74} & 89.05\textsubscript{$\pm$0.12} & 90.87\textsubscript{$\pm$0.08} & 86.40 \\
        AdaLoRA   & 87.90\textsubscript{$\pm$0.09} & 94.80\textsubscript{$\pm$0.13} & 62.41\textsubscript{$\pm$0.66} & 90.16\textsubscript{$\pm$0.07} & 92.67\textsubscript{$\pm$0.12} & 81.59\textsubscript{$\pm$0.72} & 89.38\textsubscript{$\pm$0.57} & 91.08\textsubscript{$\pm$0.11} & 86.25 \\
        PiSSA    & 87.95\textsubscript{$\pm$0.01} & 94.53\textsubscript{$\pm$0.19} & 64.66\textsubscript{$\pm$0.98} & 90.97\textsubscript{$\pm$0.05} & 92.53\textsubscript{$\pm$0.14} & 79.18\textsubscript{$\pm$0.68} & 89.79\textsubscript{$\pm$1.41} & 90.96\textsubscript{$\pm$0.08} & 86.32 \\
        MiLoRA     & 87.88\textsubscript{$\pm$0.11} & 94.69\textsubscript{$\pm$0.30} & 64.31\textsubscript{$\pm$0.97} & 91.02\textsubscript{$\pm$0.04} & 92.96\textsubscript{$\pm$0.21} & 81.35\textsubscript{$\pm$1.23} & 89.30\textsubscript{$\pm$0.12} & 90.96\textsubscript{$\pm$0.05} & 86.56 \\
        % AdaLoRA   & 87.21\textsubscript{$\pm$0.05} & 95.07\textsubscript{$\pm$0.32} & 61.37\textsubscript{$\pm$1.01} & 89.75\textsubscript{$\pm$0.10} & 92.54\textsubscript{$\pm$0.08} & 81.11\textsubscript{$\pm$0.85} & 89.05\textsubscript{$\pm$0.31} & 90.62\textsubscript{$\pm$0.11} & 85.84 \\
        % rsLoRA  & 85.26\textsubscript{$\pm$3.34} & 92.35\textsubscript{$\pm$0.33} & 65.17\textsubscript{$\pm$0.82} & 70.76\textsubscript{$\pm$10.71} & 92.48\textsubscript{$\pm$0.27} & 79.54\textsubscript{$\pm$1.33} & 89.05\textsubscript{$\pm$0.31} & 90.88\textsubscript{$\pm$0.13} & 83.19 \\
        LoRA+    & 86.96\textsubscript{$\pm$0.65} & 93.92\textsubscript{$\pm$0.00} & 63.32\textsubscript{$\pm$0.69} & 90.69\textsubscript{$\pm$0.07} & 92.77\textsubscript{$\pm$0.01} & 81.59\textsubscript{$\pm$1.18} & 88.97\textsubscript{$\pm$0.35} & 90.84\textsubscript{$\pm$0.14} & 86.13 \\
        LoRA-GA & 85.18\textsubscript{$\pm$0.16} & 93.16\textsubscript{$\pm$0.18} & 62.09\textsubscript{$\pm$0.80} & 88.57\textsubscript{$\pm$0.08} & 91.64\textsubscript{$\pm$0.13} & 76.77\textsubscript{$\pm$4.46} & 89.54\textsubscript{$\pm$0.62} & 90.87\textsubscript{$\pm$0.09} & 84.73 \\
        CorDA & 86.28\textsubscript{$\pm$0.27} & 94.38\textsubscript{$\pm$0.53} & 62.43\textsubscript{$\pm$0.94} & 90.14\textsubscript{$\pm$0.11} & 92.36\textsubscript{$\pm$0.24} & 78.10\textsubscript{$\pm$0.55} & 89.87\textsubscript{$\pm$0.93} & 90.68\textsubscript{$\pm$0.14} & 85.53 \\
        DoRA      & 87.81\textsubscript{$\pm$0.04} & 95.11\textsubscript{$\pm$0.19} & 64.23\textsubscript{$\pm$0.10} & 90.65\textsubscript{$\pm$0.11} & 92.93\textsubscript{$\pm$0.10} & 81.35\textsubscript{$\pm$0.95} & 89.54\textsubscript{$\pm$0.23} & 91.01\textsubscript{$\pm$0.16} & 86.58 \\
        \rowcolor{blue!10}\textbf{SCLoRA}   & 87.95\textsubscript{$\pm$0.12} & 95.37\textsubscript{$\pm$0.29} & 64.79\textsubscript{$\pm$0.20} & 90.76\textsubscript{$\pm$0.08} & 93.09\textsubscript{$\pm$0.14} & 83.15\textsubscript{$\pm$0.17} & 90.32\textsubscript{$\pm$0.86} & 91.22\textsubscript{$\pm$0.05} & 87.08 \\
                
        \bottomrule
\end{tabular}
\caption{Comparison of various methods on GLUE tasks with different random seeds.}

\end{table*}

\begin{table*}[t]
    \centering
    \small
    \setlength{\tabcolsep}{0.8pt}
    \label{tab:main_nlu_deberta_std}
    \begin{tabular}{l |  c | c c c c c c c c c}
        \toprule
        Method  & \# Params & MNLI & SST-2 & CoLA & QQP & QNLI & RTE & MRPC & STS-B & Avg. \\ \midrule
        Full FT      & 184M & 89.90    & 95.63    & 69.19 & 92.40    & 94.03 & 83.75 & 89.46 & 91.60 & 88.09 \\ \midrule
        BitFit       & 0.10M & 89.37    & 94.84    & 66.96 & 88.41    & 92.24 & 78.70 & 87.75 & 91.35 & 86.02 \\
        HAdapter     & 1.22M & 90.13   & 95.53    & 68.64 & 91.91    & 94.11 & 84.48 & 89.95 & 91.48 & 88.12 \\
        PAdapter     & 1.18M & 90.33   & 95.61    & 68.77 & 92.04    & 94.29 & 85.20 & 89.46 & 91.54 & 88.24 \\
        LoRA\textsubscript{r=8} & 1.33M & 90.65  & 94.95    & 69.82 & 91.99    & 93.87 & 85.20 & 89.95 & 91.60 & 88.34 \\
        AdaLoRA & 1.27M & 90.76 & 96.10 & 71.45
        & 92.23 & 94.55 & 88.09 & 90.69 & 91.84 & 89.31 \\
        \rowcolor{blue!10}\textbf{SCLoRA}  &  1.33M &  90.36\textsubscript{$\pm$0.03} & 96.33\textsubscript{$\pm$0.41}  & 71.49\textsubscript{$\pm$0.60} & 92.33\textsubscript{$\pm$0.01}  & 94.57\textsubscript{$\pm$0.05}  & 89.41\textsubscript{$\pm$0.17}  &  91.58\textsubscript{$\pm$1.21} & 92.19\textsubscript{$\pm$0.07} & 89.78 \\
        \midrule
        HAdapter     & 0.61M & 90.12    & 95.30    & 67.87 & 91.65     & 93.76 & 85.56 & 89.22 & 91.30 & 87.93 \\
        PAdapter     & 0.60M & 90.15    & 95.53    & 69.48 & 91.62    & 93.98 & 84.12 & 89.22 & 91.52 & 88.04 \\
        HAdapter     & 0.31M & 90.10    & 95.41    & 67.65 & 91.54    & 93.52 & 83.39 & 89.25 & 91.31 & 87.60 \\
        PAdapter     & 0.30M & 89.89    & 94.72    & 69.06 & 91.40    & 93.87 & 84.48 & 89.71 & 91.38 & 87.90 \\
        LoRA\textsubscript{r=2} & 0.33M & 90.30  & 94.95    & 68.71 & 91.61     & 94.03 & 85.56 & 89.71 & 91.68 & 88.15 \\
        AdaLoRA & 0.32M & 90.66 & 95.80 & 70.04
        & 91.78 & 94.49 & 87.36 & 90.44 & 91.63 & 88.86 \\
        \rowcolor{blue!10}\textbf{SCLoRA}  & 0.33M & 90.66\textsubscript{$\pm$0.02} & 96.18\textsubscript{$\pm$0.14} & 71.83\textsubscript{$\pm$1.08} & 91.82\textsubscript{$\pm$0.07} & 94.50\textsubscript{$\pm$0.00} & 89.89\textsubscript{$\pm$1.47} & 91.83\textsubscript{$\pm$0.90} & 92.00\textsubscript{$\pm$0.23} & 89.83 \\
        \bottomrule

    \end{tabular}
    \caption{Performance comparison of various methods with DeBERTaV3\textsubscript{base} on GLUE tasks with different random seeds. The results for the baselines are copied from~\citep{zhang2023adalora}.}
    
\end{table*}

\subsection{Question Answering}
\subsubsection{Dataset description}\label{app:dataset_description_qa}
We describe the benchmark dataset of SQuAD~\citep{rajpurkar2016squad,rajpurkar2018squadv2}. The Stanford Question Answering Dataset (SQuAD) is a benchmark for reading comprehension, featuring questions based on Wikipedia articles. Each question is answered with a specific text segment (or span) from the corresponding passage, though some questions may have no answer at all.

\begin{itemize}
    \item \textbf{SQuADv1.1.}~\footnote{\url{https://huggingface.co/datasets/rajpurkar/squad}} Over 100,000 question-answer pairs derived from more than 500 articles. The dataset consists of 87,599 samples for training and 10,570 for validation.
    \item \textbf{SQuADv2.0.}~\footnote{\url{https://huggingface.co/datasets/rajpurkar/squad_v2}} Combines the 100,000 questions in SQuADv1.1 with over 50,000 unanswerable questions to closely resemble answerable ones. To perform well on SQuADv2.0, systems must not only provide correct answers when available but also recognize when a question cannot be answered based on the given passage and abstain from responding. The dataset consists of 130,319 samples for training and 11,873 for validation.
\end{itemize}

\subsubsection{Experimental setup}\label{app:setup_qa}

We evaluate \textbf{SCLoRA} on two question answering (QA) tasks: SQuAD v1.1~\citep{rajpurkar2016squad} and SQuADv2.0~\citep{rajpurkar2018squadv2}. Following~\citet{zhang2023adalora}, we fine-tune a pre-trained DeBERTaV3\textsubscript{base}~\citep{he2021debertav3} with \textbf{SCLoRA}  and set the rank $r$ of LoRA as $\{1,2,4,8\}$. These tasks are considered as a sequence labeling problem, where the goal is to predict the probability of each token being the start and end of the answer span. We measured the performance of model using the Exact Match (EM) and F1 metrics. We use 1 GPU of NVIDA RTX 3090 24GB for experiments. We conduct the experiments in Huggingface Framework.

\subsubsection{Hyperparameters}\label{app:best_hyperparameter_qa}
\begin{table}[h]
\label{tab:parmeter_qa}
\centering
\small
\setlength{\tabcolsep}{5pt}
\begin{tabular}{l  c  ccc} \toprule
Dataset & $r$ & Learning rate & $q$ & $\gamma$ \\ \midrule
\multirow{4}{*}{SQuADv1.1} & 1 & \num{2e-3} & 0.5  &  \num{5e-1} \\
                           & 2 & \num{2e-3} &  0.5 & \num{1e-1}  \\
                           & 4 & \num{1e-3}  & 0.5 & \num{1e-2} \\
                           & 8 & \num{2e-3} &  0.5 & \num{1e-1} \\ \midrule
\multirow{4}{*}{SQuADv2.0} & 1 & \num{2e-3} & 0.75 & \num{7e-2} \\
                           & 2 &\num{2e-3}  & 0.75 & \num{5e-2} \\
                           & 4 &\num{3e-3}   & 0.75 & \num{5e-1} \\ 
                           & 8 & \num{3e-3}& 0.75 & \num{5e-1}  \\ \bottomrule
\end{tabular}
\caption{Best hyperparameters for \textbf{SCLoRA} in question answering}

\end{table}

To tune \textbf{SCLoRA}, we fix the batch size as 16, and train 10 epochs for SQuADv1.1 and 12 epochs for SQuADv2.0, respectively. We search for the learning rate from $\{\num{1e-3}, \num{2e-3}, \num{3e-3}\}$, $q$ from $\{0.5, 0.75, 1.0\}$ and $\gamma$ from $ \num{5e-1}, \num{1e-1}, \num{7e-2}, \num{5e-2}, \num{1e-2}\}$. The learnable singular vectors $U/V$ are be initialized as $U$ with zeros and $V$ with random Gaussian initialization. We report the best hyperparameters of \textbf{SCLoRA} in \Cref{tab:parmeter_qa}.

\subsection{Commonsense Reasoning}
\subsubsection{Dataset description}\label{app:dataset_description_cr}
The commonsense reasoning tasks are intended to require the model to go beyond pattern recognition. Instead, the model should use ``common sense'' or world knowledge to make inferences. The commonsense reasoning tasks comprise 8 sub-tasks, each with a predefined training and testing set. 

\begin{itemize}
    \item \textbf{BoolQ.} The model answers yes/no questions about short passages, testing its ability to understand statements.
    \item \textbf{PIQA.} The model chooses the most plausible solution for a physical interaction, focusing on practical reasoning.
    \item \textbf{SIQA.} The model infers the most suitable outcome or rationale in everyday social contexts.
    \item \textbf{HellaSwag.} The model selects the most coherent continuation of a short scenario, emphasizing commonsense inference.
    \item \textbf{WinoGrande.} The model resolves ambiguous pronoun references that require broad commonsense to disambiguate.
    \item \textbf{ARC-e.} The model tackles elementary-level science questions assessing basic scientific knowledge.
    \item \textbf{ARC-c.} The model addresses harder, more nuanced science questions requiring deeper reasoning.
    \item \textbf{OBQA.} OpenBookQA. The model answers questions using a provided `open book' of facts, testing its ability to integrate and apply specific knowledge.
    
\end{itemize}

\subsubsection{Experimental setup}\label{app:setup_cr}
Following ~\citet{hu2023llmadapters}, we amalgamate the training datasets from all 8 tasks to create the final training dataset and conduct evaluations on the individual testing dataset for each task. 

\subsubsection{Hyperparameters}\label{app:best_hyperparameter_cr}

To tune \textbf{SCLoRA},  we search for the learning rate from $\{\num{3e-4}, \num{4e-4}\}$, $q$ from $\{0.5, 0.75, 1.0\}$ and $\gamma$ from $\{\num{5e-1}, \num{1e-1}\}$. The learnable singular vectors $U/V$ are initialized as $U$ with zeros, $V$ with random Gaussian initialization. We report the best hyperparameters of \textbf{SCLoRA} in \Cref{tab:parmeter_cr}. 

\begin{table}[h]
\label{tab:parmeter_cr}
\centering
\small
\setlength{\tabcolsep}{2.5pt}
\begin{tabular}{l ccccccc} \toprule
Model & Learning rate & Batch size & \#Epochs & $q$ & $\gamma$ \\ \midrule
LLaMA-7B & \num{3e-4} & 16 & 3  & 0.5 & 0.1 \\
LLaMA2-7B & \num{4e-4} & 16 & 3  & 0.75  & 0.5 \\
% LLaMA-7B & \num{2e-4} & 16 & 3  & 0.5 & \num{5e-1}  \\
% LLaMA2-7B & \num{3e-4} & 16 & 3 & 1.0 & \num{1e-2}  \\
 \bottomrule
\end{tabular}
\caption{Best hyperparameters for \textbf{SCLoRA} in commonsense reasoning}

\end{table}

\section{Spectral analysis of $\Delta W$}\label{sec:spectral_analysis}

\begin{figure}[t]
    \centering
    \subfigure[Spectral Norm]{\includegraphics[width=0.49\columnwidth]{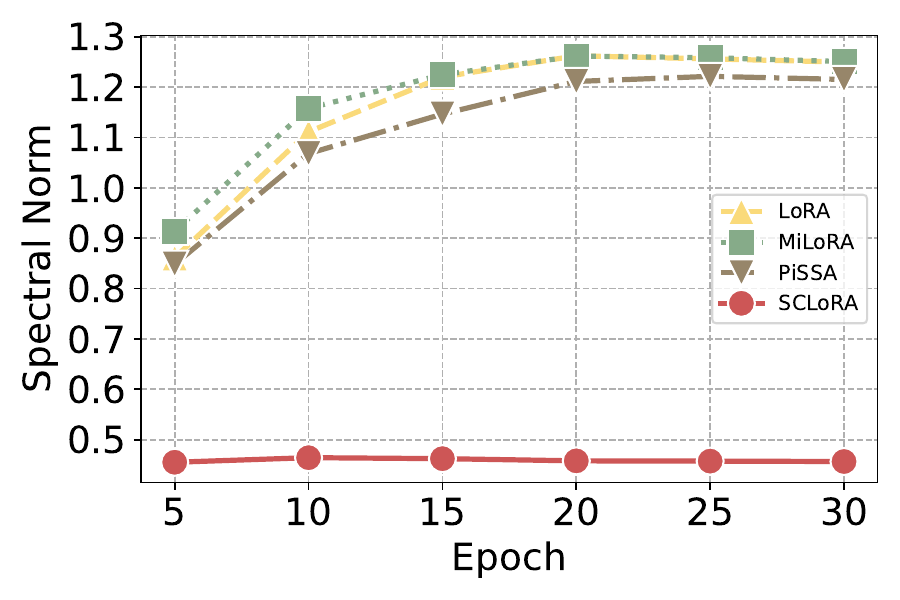}}
    \subfigure[Accuracy]{\includegraphics[width=0.49\columnwidth]{fig/5.Discussion/acc_drop/mrpc_acc.pdf}}\hfill
    % \subfigure[Evaluation Loss]{\includegraphics[width=0.31\textwidth]{fig/ICLR26-revision_figure/acc_drop/mrpc_loss.pdf}}
    \caption{Changes during fine-tuning RoBERTa\textsubscript{base} on the MRPC dataset: (a) Spectral norm of $\Delta W$ in the query layer; (b) accuracy on the pre-trained task (BookCorpus dataset).}
    \label{fig:catastrophic_forgetting}
\end{figure}

\Cref{pro:spectral_norm} demonstrates that the spectral norm of large weight matrices increases rapidly when adaptive optimizers are applied. However, \textbf{SCLoRA} prevents this increase by restricting the range of the singular value of $\Delta W$ during fine-tuning, ensuring that the spectral norm does not grow excessively. To empirically verify this difference, \Cref{fig:catastrophic_forgetting} (a) illustrates the evolution of the spectral norm of $\Delta W$ across various methods during the fine-tuning with MRPC data on RoBERTa\textsubscript{base}. While LoRA and its variants tend to increase the spectral norm during fine-tuning, \textbf{SCLoRA} maintains a smaller spectral norm. This suggests that, unlike other LoRA-based methods, \textbf{SCLoRA} adapts to fine-tuning tasks by effectively incorporating new information with the singular components with spectral clipping.
Furthermore, consistent with \Cref{thm:map}, we show that the uncontrolled growth of adapter singular values induces catastrophic forgetting, which is empirically confirmed in \Cref{fig:catastrophic_forgetting}(b). As fine-tuning progresses, the performance of baseline methods degrades sharply, whereas \textbf{SCLoRA} preserves approximately three-times higher performance, demonstrating its strong robustness against catastrophic forgetting.

\section{Further Experiments on Mitigating Catastrophic Forgetting}\label{app:sec_cf}

In this section, we validate \textbf{SCLoRA} on mitigating the catastrophic forgetting, followed by \Cref{sec:discussion}.

\subsection{Mitigating catastrophic forgetting with other models}

\begin{table*}[t]
\centering
\small
\begin{tabular}{l l cc cc cc cc cc}
\toprule
\multirow{2}{*}{Dataset} & \multirow{2}{*}{Method} 
& \multicolumn{2}{c}{SST-2} 
& \multicolumn{2}{c}{CoLA} 
& \multicolumn{2}{c}{RTE} 
& \multicolumn{2}{c}{QNLI} 
& \multicolumn{2}{c}{MRPC} \\
\cmidrule(lr){3-4}
\cmidrule(lr){5-6}
\cmidrule(lr){7-8}
\cmidrule(lr){9-10}
\cmidrule(lr){11-12}
& & Acc.\textsubscript{ft} & Acc.\textsubscript{pt} 
  & Acc.\textsubscript{ft} & Acc.\textsubscript{pt}
  & Acc.\textsubscript{ft} & Acc.\textsubscript{pt}
  & Acc.\textsubscript{ft} & Acc.\textsubscript{pt}
  & Acc.\textsubscript{ft} & Acc.\textsubscript{pt} \\
\midrule
    
\multirow{7}{*}{BookCorpus}
& Pre-trained & - & 61.64 & - & 61.64 & - & 61.64 & - & 61.64  & - & 61.64 \\
    & LoRA & 94.80 & 32.35 & 64.49 & 50.85 & 92.73 & 56.76 & 80.39 & 33.78 & 89.05 & 3.77 \\
    & PiSSA & 94.53 & 27.14 & 64.66 & 44.93 & 92.53 & 52.75 & 79.18 & 10.85 & 89.79 & 4.78 \\
    & MiLoRA & 94.69 & 27.24 & 64.31 & 51.38 & 92.96 & 52.67 & 81.35 & 25.99 & 89.30 & 2.50 \\
    & LoRA-GA & 93.16 & 0.11 & 62.09 & 18.47 & 91.64 & 0.26 & 76.77 & 2.94 & 89.54 & 0.03 \\
    & CorDA & 94.38 & 29.38 & 62.43 & 29.57 & 92.36 & 56.12 & 78.10 & 33.21 & 89.87 & 14.26 \\
    \rowcolor{blue!10}\cellcolor{white} & \textbf{SCLoRA} & \textbf{95.37} & \textbf{51.29} &\textbf{ 64.79} & \textbf{54.78} & \textbf{93.09} & \textbf{57.15} & \textbf{83.15} & \textbf{50.07} & \textbf{90.32} & \textbf{32.00 }\\
\midrule

\multirow{7}{*}{OpenWebText}
& Pre-trained & - & 68.21 & - & 68.21 & - & 68.21 & - & 68.21& - & 68.21 \\
    & LoRA & 94.80 & 54.33 & 64.49 & 63.67 & 92.73 & 62.88 & 80.39 & 53.47 & 89.05 & 18.36 \\
    & PiSSA & 94.53 & 48.49 & 64.66 & 60.28 & 92.53 & 58.96 & 79.18 & 45.97 & 89.79 & 17.92 \\
    & MiLoRA & 94.69 & 49.22 & 64.31 & 64.25 & 92.96 & 60.58 & 81.35 & 51.74 & 89.30 & 16.93 \\
    & LoRA-GA & 93.16 & 0.27 & 62.09 & 26.78 & 91.64 & 0.95 & 76.77 & 7.81 & 89.54 & 0.04 \\
    & CorDA & 94.38 & 50.36 & 62.43 & 59.56 & 92.36 & 63.37 & 78.10 & 54.07 & 89.87 & 40.10 \\    \rowcolor{blue!10}\cellcolor{white} &\textbf{SCLoRA} & \textbf{95.37} & \textbf{65.67} &\textbf{64.79} & \textbf{64.31} & \textbf{93.09} & \textbf{63.59} & \textbf{83.15} & \textbf{62.69} &\textbf{ 90.32} & \textbf{47.27 }\\
\midrule

\multirow{7}{*}{STORIES}
& Pre-trained & - & 69.49 & - & 69.49 & - & 69.49 & - & 69.49& - & 69.49 \\
    & LoRA & 94.80 & 40.13 & 64.49 & 61.87 & 92.73 & 62.89 & 80.39 & 45.90 & 89.05 & 8.42 \\
    & PiSSA & 94.53 & 37.27 & 64.66 & 59.20 & 92.53 & 57.15 & 79.18 & 31.86 & 89.79 & 7.75 \\
    & MiLoRA & 94.69 & 37.47 & 64.31 & 61.46 & 92.96 & 60.30 & 81.35 & 36.13 & 89.30 & 4.28 \\
    & LoRA-GA & 93.16 & 0.14 & 62.09 & 24.80 & 91.64 & 0.36 & 76.77 & 4.75 & 89.54 & 0.04 \\
    & CorDA & 94.38 & 34.17 & 62.43 & 52.48 & 92.36 & 63.41 & 78.10 & 53.56 & 89.87 & 23.46 \\
    \rowcolor{blue!10}\cellcolor{white}  & \textbf{SCLoRA} & \textbf{95.37} & \textbf{60.38} & \textbf{64.79} & \textbf{64.45} & \textbf{93.09} & \textbf{63.91} & \textbf{83.15} & \textbf{60.31} & \textbf{90.32} & \textbf{41.12} \\
\bottomrule
\end{tabular}
\caption{Comparison on catastrophic forgetting with RoBERTa as the backbone model. `Acc.\textsubscript{ft}' and `Acc.\textsubscript{pt}' denote the accuracies on the fine-tuning and pre-training tasks, respectively.}
\label{app:catastrophic_roberta_all}
\end{table*}

To further validate the effectiveness of \textbf{SCLoRA} in mitigating catastrophic forgetting, we report the performance of RoBERTa\textsubscript{base} models trained on SST-2, CoLA, RTE, and MRPC, along with the accuracy on their corresponding pre-training tasks using BookCorpus~\citep{zhu2015bookcorpus}, OpenWebText~\citep{Gokaslan2019OpenWebText}, and STORIES~\citep{trinh2018stories} datasets in \Cref{app:catastrophic_roberta_all}. Although the degree of catastrophic forgetting varies across existing LoRA models depending on the type of data, \textbf{SCLoRA} consistently adapts to new downstream tasks with improved fine-tuning performance while substantially preserving accuracy on the pre-training tasks. In particular, across all datasets, \textbf{SCLoRA} maintains much higher pre-training accuracy than LoRA while still achieving the best fine-tuned task performance.
Although MiLoRA modifies the minor components with the aim of preserving pre-trained knowledge, it still suffers from catastrophic forgetting during fine-tuning due to uncontrolled spectral growth. In addition, both CorDA and LoRA-GA exhibit even more severe forgetting, as they involve uncontrolled spectral updates and initialize the adapters using a subset of calibration samples from the fine-tuning task. This initialization is biased toward the downstream distribution, making the model more prone to drifting away from the pre-trained knowledge. In particular, LoRA-GA is designed to approximate the full gradient of the pre-trained weights on the fine-tuning task, and prior studies have shown that full-gradient adaptation tends to induce even greater forgetting~\citep{biderman2024loralearnless}, which explains its more pronounced degradation of pre-trained knowledge.
These results demonstrate that \textbf{SCLoRA} effectively mitigates catastrophic forgetting without sacrificing downstream effectiveness.

\subsection{Catastrophic forgetting with continual Learning}

While our primary analysis focuses on single-task fine-tuning scenarios, catastrophic forgetting is inherently a continual learning problem. To more rigorously evaluate the effectiveness of \textbf{SCLoRA} in achieving stable adaptation while mitigating catastrophic forgetting, we therefore conduct additional experiments in continual learning settings. Specifically, we consider both a GLUE-based sequential fine-tuning setup and a standardized continual learning benchmark proposed in prior work.

\subsubsection{Continual learning on GLUE task}

\begin{table*}[h!]
\centering
\small
\begin{tabular}{lcccccccc}
\toprule
\multirow{2}{*}{Method} &
\multicolumn{2}{c}{SST-2} &
\multicolumn{2}{c}{CoLA} &
\multicolumn{2}{c}{RTE} &
\multicolumn{2}{c}{MRPC} \\
\cmidrule(lr){2-3}
\cmidrule(lr){4-5}
\cmidrule(lr){6-7}
\cmidrule(lr){8-9}
&
{Acc.\textsubscript{ft}} &
{Acc.\textsubscript{pt}} &
{Acc.\textsubscript{ft}} &
{Acc.\textsubscript{pt}} &
{Acc.\textsubscript{ft}} &
{Acc.\textsubscript{pt}} &
{Acc.\textsubscript{ft}} &
{Acc.\textsubscript{pt}} \\
\midrule
Pre-train & -- & 61.64 & -- & 61.64 & -- & 61.64 & -- & 61.64 \\
\midrule
LoRA & 94.50 & 45.15 & 59.74 & 48.51 & 71.00 & 40.47 & 88.97 & 33.38 \\
MiLoRA & 94.76 & 48.64 & 59.77 & 52.64 & 71.00 & 47.33 & 88.24 & 29.77 \\
\rowcolor{blue!10} \textbf{SCLoRA} & \textbf{95.37} & \textbf{51.76} & \textbf{60.77} & \textbf{54.04} & \textbf{79.18} & \textbf{52.44} & \textbf{89.30} & \textbf{45.08} \\
\bottomrule
\end{tabular}
\caption{Continual learning on GLUE. RoBERTa\textsubscript{base} are sequentially fine-tuned on SST-2 $\rightarrow$ CoLA $\rightarrow$ RTE $\rightarrow$ MRPC, and after each task we report both downstream accuracy and pre-training accuracy on BookCorpus.}
\end{table*}

In the GLUE-based continual learning experiments, we sequentially fine-tune $\Delta W$ across multiple downstream tasks. For each task, hyperparameters such as the number of training epochs and batch size follow the baseline and reported optimal settings to ensure a fair comparison. After completing fine-tuning on each task, we evaluate both the performance on the corresponding downstream task and the accuracy on the pretraining dataset, BookCorpus. Compared with LoRA and the SVD-based method MiLoRA, \textbf{SCLoRA} exhibits a substantially slower degradation of pretraining performance while consistently achieving stronger adaptation to downstream tasks throughout the sequence. These results indicate that \textbf{SCLoRA} effectively preserves pretrained knowledge even when learning multiple tasks sequentially, enabling more stable and reliable task adaptation.

\subsubsection{Continual learning with benchmark}

\begin{table}[t]
\centering
\small
\begin{tabular}{lcccc}
\toprule
Method & Order-1 & Order-2 & Order-3 & Avg \\
\midrule
SeqFT   & 18.9 & 24.9 & 41.7 & 28.5 \\
SeqLoRA & 44.6 & 32.7 & 53.7 & 43.7 \\
IncLoRA & 66.0 & 64.9 & 68.3 & 66.4 \\
Replay  & 55.2 & 56.9 & 61.3 & 57.8 \\
EWC     & 48.7 & 47.7 & 54.5 & 50.3 \\
LwF     & 54.4 & 53.1 & 49.6 & 52.3 \\
L2P     & 60.3 & 61.7 & 61.1 & 60.7 \\
LFPT5   & 67.6 & 72.6 & \textbf{77.9} & 72.7 \\
O-LoRA  & 75.4 & 75.7 & 76.3 & 75.8 \\
\rowcolor{blue!10} \textbf{SCLoRA} &
\textbf{76.9}\textsubscript{$\pm$1.0}&\textbf{76.8}\textsubscript{$\pm$1.9}&76.2\textsubscript{$\pm$0.9} &\textbf{76.6}\\ 
\bottomrule
\end{tabular}
\caption{Standard CL Benchmark}

\end{table}

In addition, to assess performance under a more standardized continual learning protocol, we follow the continual learning benchmark introduced in O-LoRA~\citep{wang2023olora}. We strictly follow the prescribed experimental protocol and evaluate three different task orders using the T5 model on five datasets — AG News, Amazon Reviews, Yelp Reviews, DBpedia, and Yahoo Answers — repeating each experiment three times. To ensure a fair comparison with O-LoRA in terms of parameter budget, we train a new adapter for each task while applying \textbf{SCLoRA} under the same parameter configuration. On average, \textbf{SCLoRA} outperforms existing methods in the continual learning setting on average. This further demonstrates that the spectral constraints imposed by \textbf{SCLoRA} effectively mitigate catastrophic forgetting not only in single-task fine-tuning but also in continual learning scenarios.

\section{Additional Studies}\label{app:additional_studies}

 \begin{figure}[h]
    \centering
    \subfigure[SQuADv1.1, $r$=4]{\includegraphics[width=0.49\columnwidth]{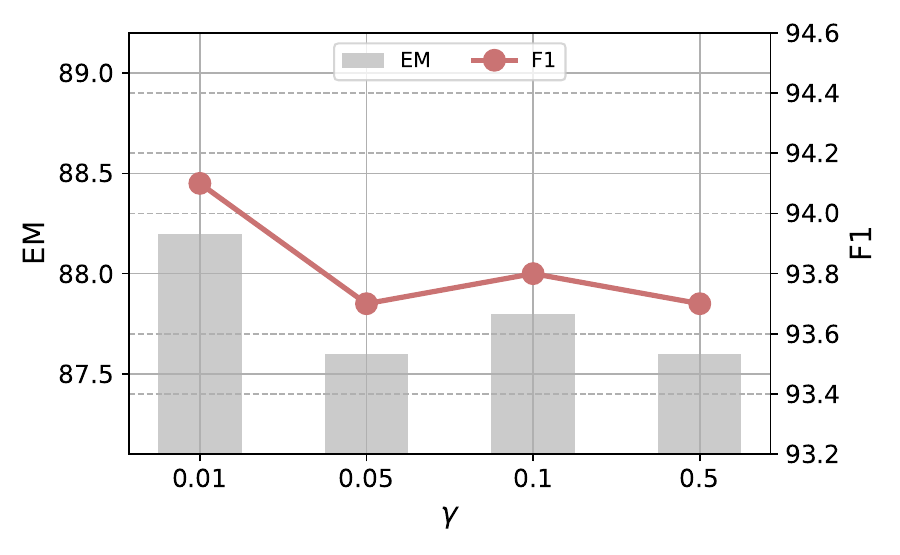}}
    \subfigure[SQuADv2.0, $r$=4]{\includegraphics[width=0.49\columnwidth]{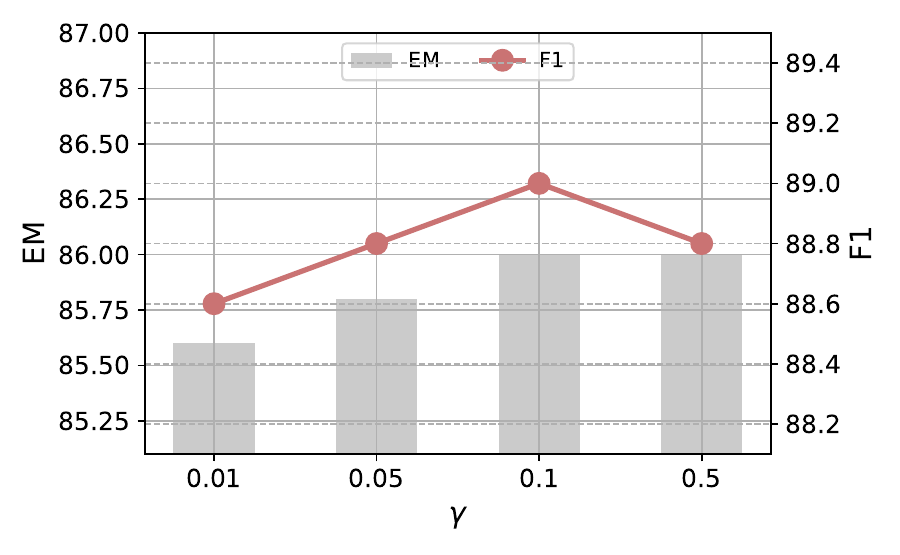}}
    \caption{Sensitivity studies on $\gamma$}
    \label{fig:sens_gamma_sigma}
\end{figure}

\subsection{Sensitivity study on the orthogonal regularization coefficient $\gamma$}\label{app:sens_ortho}
The orthogonal regularization applied to $U$ and $V$ is used to learn the singular values that consists the injected singular components. We conduct a sensitivity study on the orthogonal regularization coefficient $\gamma$ on Question-Answering (QA) tasks by fine-tuning the DeBERTaV3\textsubscript{base} model on the SQuAD v1.1/v2.0 datasets with rank $r=4$. As shown in~\Cref{fig:sens_gamma_sigma}, the optimal value of $\gamma$ varies across datasets, indicating that different tasks prefer different regularization strengths. However, the performance does not degrade sharply outside the optimal range; instead, the model remains largely robust, exhibiting no severe sensitivity to the choice of $\gamma$.

\subsection{Sensitivity on the rank $r$}
To evaluate how the choice of rank affects both downstream performance and catastrophic forgetting, we conduct a sensitivity study by varying the rank $r$. We measure performance using RoBERTa on two representative GLUE tasks—MRPC and STS-B—and assess the corresponding degradation on the BookCorpus pre-training dataset. The results are reported in \Cref{fig:sens_r}. Although the optimal adapter rank and configuration vary across datasets and model sizes, and increasing the rank does not always yield consistent gains, we generally do not observe drastic performance degradation as the rank increases.

\begin{figure}[t]
    \centering
    \subfigure[MRPC]{\includegraphics[width=0.48\columnwidth]{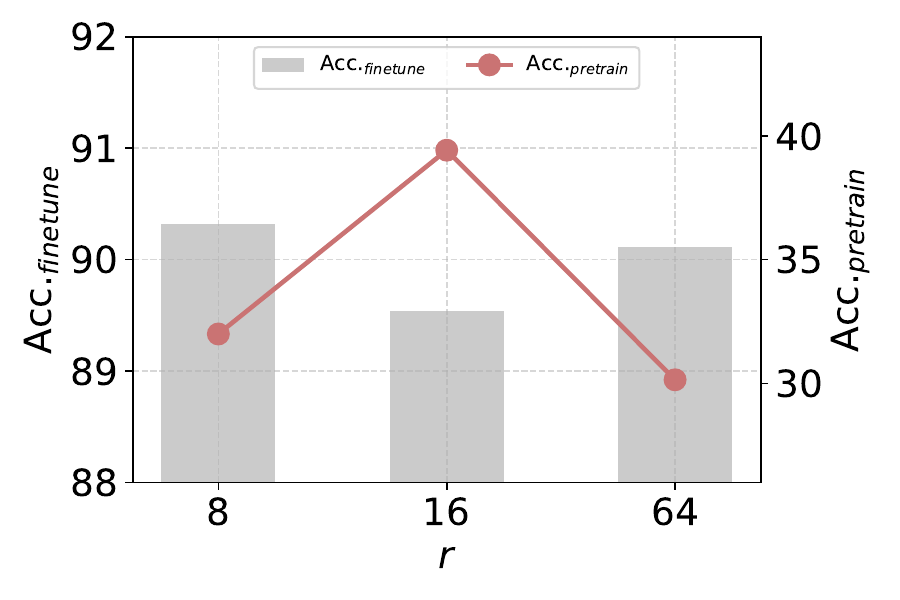}}
    \subfigure[STS-B]{\includegraphics[width=0.48\columnwidth]{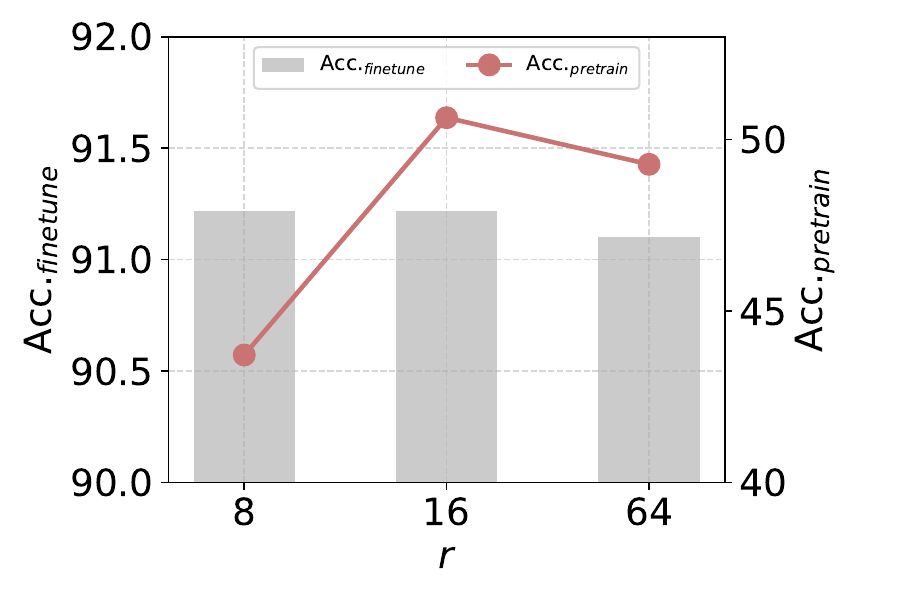}}
    \caption{Sensitivity on $r$}
    \label{fig:sens_r}
\end{figure}

\subsection{Additional sensitivity on $q$}

\begin{figure}[t]
    \centering
    \small
    \begin{subfigure}[SQuADv1.1, $r$=4]{\includegraphics[width=0.48\columnwidth]{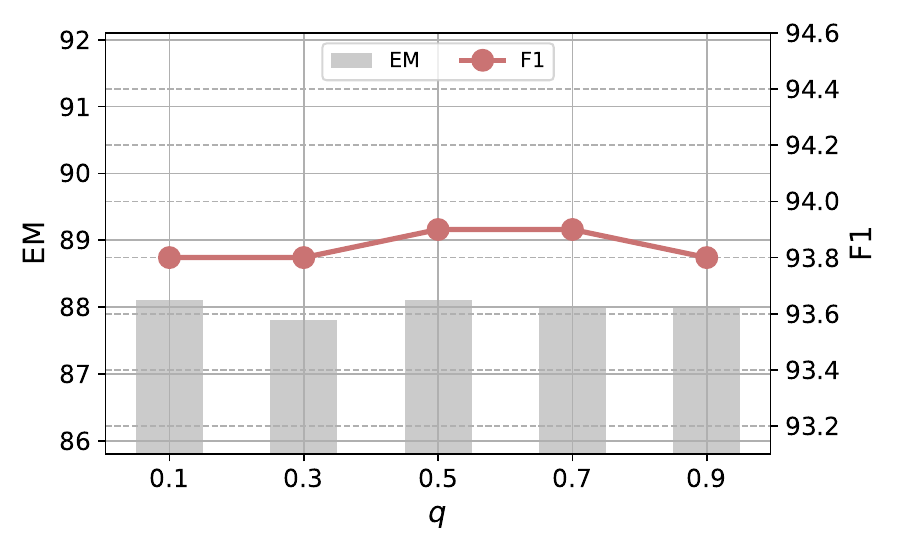}}
    \end{subfigure}
    \begin{subfigure}[SQuADv1.1, $r$=8]{\includegraphics[width=0.48\columnwidth]{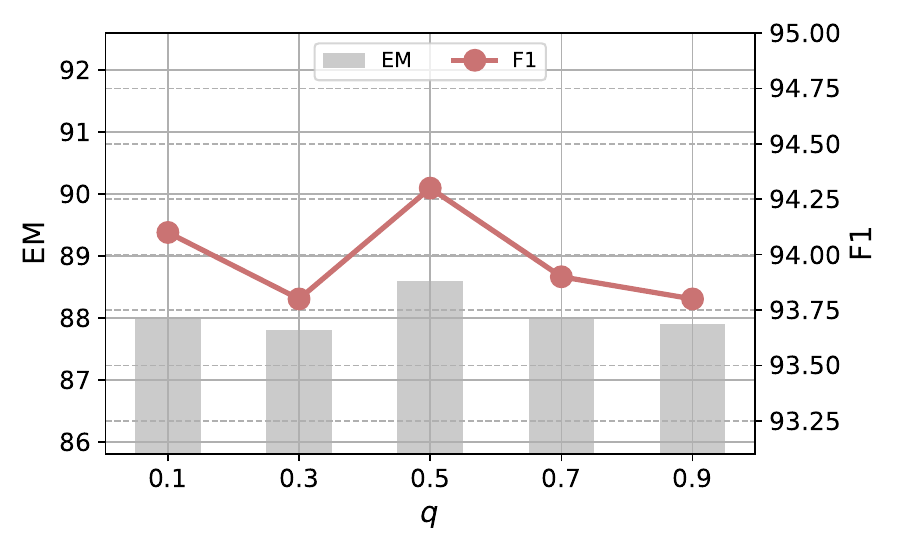}}
    \end{subfigure}
    \begin{subfigure}[SQuADv2.0, $r$=4]{\includegraphics[width=0.48\columnwidth]{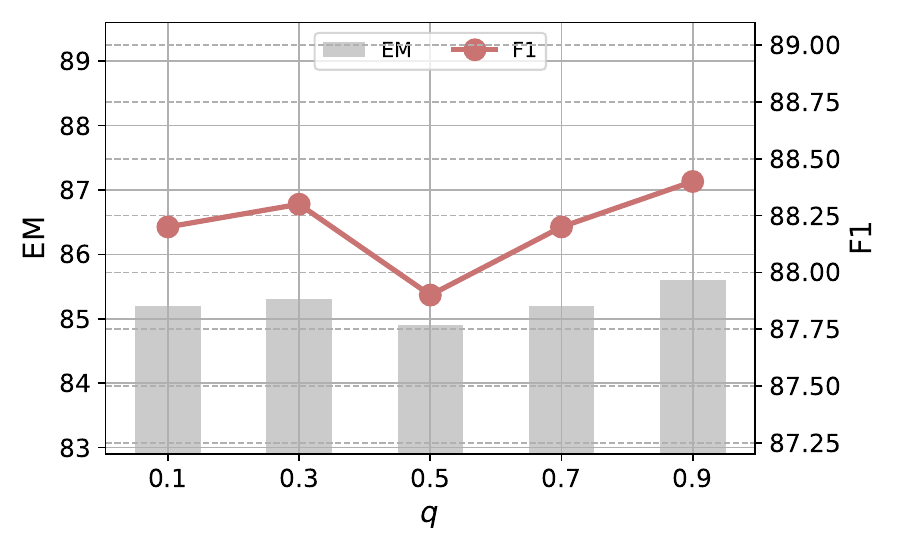}}
    \end{subfigure}
    \begin{subfigure}[SQuADv2.0, $r$=8]{\includegraphics[width=0.48\columnwidth]{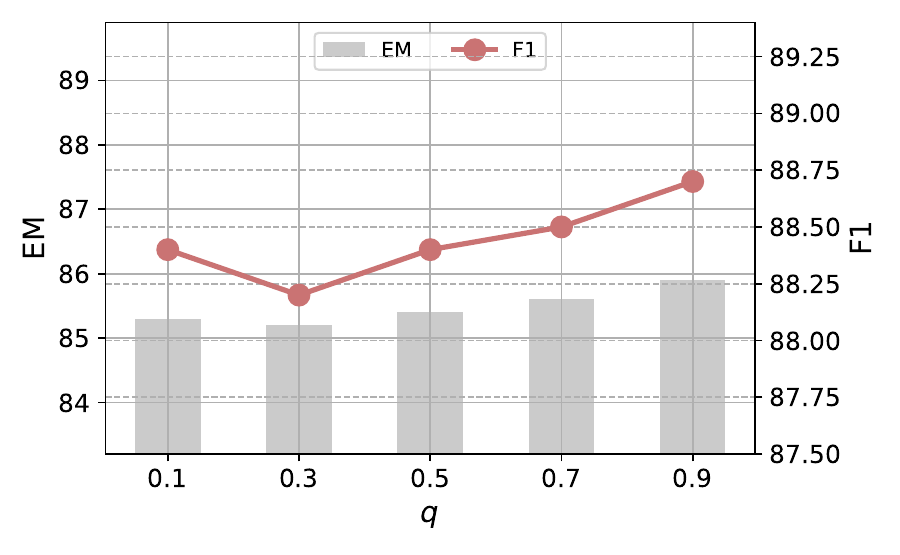}}
    \end{subfigure}
    \caption{Sensitivity on $q$}
    \label{fig:sens_quantile_squad}
\end{figure}

Followed by \Cref{sec:sens_sigma_bar}, we conduct an additional sensitivity study on the $q$ on Question-Answering (QA) tasks. \Cref{fig:sens_quantile_squad} illustrates that the fine-tuning performance does not change drastically as the quantile value varies. This indicates that using quantile-based bounds does not significantly degrade the performance of downstream tasks.

\subsection{Ablation on integrating \textbf{SCLoRA} with SVD-based LoRA}
Existing SVD-based LoRA methods either apply SVD during the initialization stage of fine-tuning or use parameterized SVD to adjust the rank. However, these approaches generally do not explicitly control the spectral dynamics during the fine-tuning process. To examine whether SCLoRA can complement such methods by addressing this aspect, we conduct an ablation study in which our framework is integrated into existing baseline models.
Specifically, for MiLoRA, we apply orthogonality regularization and singular value clipping after constructing the adapter using minor components, and for AdaLoRA, we incorporate singular value clipping within each allocated rank during the adaptive rank allocation procedure. All experiments use the best hyperparameters for both the backbone models and SCLoRA, and the results are reported in \Cref{tab:abla_integrated}.
Across downstream tasks, we observe consistent and modest performance improvements when combining SCLoRA with existing SVD-based methods (MiLoRA and AdaLoRA), even without additional hyperparameter tuning. These results suggest that spectral clipping may facilitate more effective adaptation in spectral regions that are not fully exploited in prior approaches.

\begin{table*}[t]
\centering
\small
\setlength{\tabcolsep}{7pt}
\begin{tabular}{l | cccc} \toprule
% Method & MNLI  & SST-2  & CoLA  & QQP  & QNLI  & RTE  & MRPC  & STS-B & Avg. \\ \midrule
    Method  &   CoLA &  QNLI  & MRPC & STS-B \\ \midrule
    LoRA  & 64.49\textsubscript{$\pm$0.64} & 92.73\textsubscript{$\pm$0.11} & 89.05\textsubscript{$\pm$0.12} & 90.87\textsubscript{$\pm$0.08}  \\
    \midrule
    AdaLoRA   & 62.41\textsubscript{$\pm$0.66} &92.67\textsubscript{$\pm$0.12}   & 89.38\textsubscript{$\pm$0.57} & 91.08\textsubscript{$\pm$0.11}  \\
    AdaLoRA+\textbf{SCLoRA}  & 64.56\textsubscript{$\pm$0.12} & 92.79\textsubscript{$\pm$0.17}  & 89.54\textsubscript{$\pm$0.99} & 91.08\textsubscript{$\pm$0.30} \\
    % AdaLoRA   & 61.37\textsubscript{$\pm$1.01} & 92.54\textsubscript{$\pm$0.08} & 81.11\textsubscript{$\pm$0.85} & 89.05\textsubscript{$\pm$0.31} & 90.62\textsubscript{$\pm$0.11}  \\
    % AdaLoRA+\textbf{SCLoRA}     & 64.31\textsubscript{$\pm$0.55} & 92.99\textsubscript{$\pm$0.16} & 82.67\textsubscript{$\pm$0.96} & 89.79\textsubscript{$\pm$0.79} & 91.02\textsubscript{$\pm$0.11} \\
    \midrule
    MiLoRA     & 64.31\textsubscript{$\pm$0.97} &92.96\textsubscript{$\pm$0.21} & 89.30\textsubscript{$\pm$0.12} & 90.96\textsubscript{$\pm$0.05}  \\
    MiLoRA+\textbf{SCLoRA}  & 64.57\textsubscript{$\pm$0.87} & 93.01\textsubscript{$\pm$0.11} & 89.30\textsubscript{$\pm$0.28} & 91.00\textsubscript{$\pm$0.04}  \\
    \midrule
    \rowcolor{blue!10}\textbf{SCLoRA}   & 64.79\textsubscript{$\pm$0.20} & 93.09\textsubscript{$\pm$0.14} & 90.32\textsubscript{$\pm$0.86} & 91.22\textsubscript{$\pm$0.05} \\
                
        \bottomrule
\end{tabular}
\caption{Comparison of various methods on GLUE tasks with different random seeds.}
\label{tab:abla_integrated}

\end{table*}

\subsection{Gram-Schmidt orthogonalization on parameterized singular vectors}

\begin{table*}[t]
\centering
\small
\setlength{\tabcolsep}{6.5pt}
\resizebox{0.95\textwidth}{!}{%
\begin{tabular}{lcccccccc}
\toprule
 & \multicolumn{2}{c}{CoLA} & \multicolumn{2}{c}{RTE} & \multicolumn{2}{c}{MRPC} & \multicolumn{2}{c}{STS-B} \\
 \cmidrule(lr){2-3}
 \cmidrule(lr){4-5}
 \cmidrule(lr){6-7}
 \cmidrule(lr){8-9}
 & Acc. & Min/Epoch  & Acc. &  Min/Epoch & Acc. & Min/Epoch  & Acc. &  Min/Epoch \\
\midrule
Gram-Schmidt ortho. & 57.31\textsubscript{$\pm$0.85} & 3.6 & 67.27\textsubscript{$\pm$1.23} & 1.1 & 87.58\textsubscript{$\pm$ 1.02} & 2.0 & 89.88\textsubscript{$\pm$0.19} & 2.6 \\
 \rowcolor{blue!10}\textbf{SCLoRA} & \textbf{64.79}\textsubscript{$\pm $0.20} & \textbf{2.9} & \textbf{83.15}\textsubscript{$\pm$ 0.17} &  \textbf{0.8} & \textbf{90.32}\textsubscript{$\pm$ 0.86} & \textbf{1.3} & \textbf{91.22}\textsubscript{$\pm$0.05} & \textbf{2.1} \\
\bottomrule
\end{tabular}
}
\caption{Comparison of Gram-Schmidt orthogonalization and orthogonal regularization across various datasets. `Min/Epoch' denotes the number of minutes required per epoch.}

\label{tab:gs_ortho}
\end{table*}

Orthogonal regularization is a widely used method for implementing parameterized SVD in LoRA~\citep{zhang2023adalora,cao2024sorsa, zhang2024lora2}. Unlike orthogonal regularization, Gram-Schmidt orthogonalization—one of the orthogonalization methods—produces strictly orthogonal vectors. However, since it refines each subsequent vector based on the previously orthogonalized ones, it has the disadvantage of being difficult to parallelize compared to the regularization approach. Furthermore, as shown in~\Cref{fig:ortho_loss}, the orthogonal error rapidly decreases to below $10^{-2}$ in the early stages of fine-tuning, indicating that the singular vectors can be sufficiently trained. We conduct an ablation study with orthogonal regularization by replacing it with Gram-Schmidt-based orthogonalization. We measured the performance and training speed in~\Cref{tab:gs_ortho}.

\begin{table}[t]
\centering
\small
\begin{tabular}{lcccc}
\toprule
Method & CoLA & RTE & MRPC & STS-B \\
\midrule
Descending & 63.03 & 83.03 & 89.87 & 90.84 \\
Ascending & 63.40 & 80.99 & 88.64 & 91.06 \\
\rowcolor{blue!10}\textbf{SCLoRA}  & \textbf{64.79} & \textbf{83.15} & \textbf{90.32} & \textbf{91.22} \\
\bottomrule
\end{tabular}
\caption{Ablation study on layerwise $q$}
\label{tab:abla_layerwise_quantile}
\end{table}

\subsection{Ablation on layer-wise quantile}
While the upper-bound quantile $q$ of the singular values can be applied globally across all layers, it is also possible to assign different quantiles to different groups of layers. To investigate this flexibility, we perform an ablation study in which we group the $L$ layers and assign distinct quantiles to each group.
Specifically, in the `Descending' setting, the $L$ layers are divided into four equal groups, and the quantile is assigned in decreasing order—from $q=1.0$ to $q=0.25$, which is quartile of pre-trained spectrum for simplicity—as the layer index increases. Conversely, in the `Ascending' setting, the quantile assignment is reversed, increasing from $q=0.25$ to $q=1.0$. \textbf{SCLoRA} applies a single uniform quantile across all layers, which corresponds to our default design.
As shown in~\Cref{tab:abla_layerwise_quantile} the relative performance differences across layer-wise quantile schedules are minor: the choice between Ascending and Descending varies slightly across datasets, and neither configuration leads to severe degradation. Importantly, the Global quantile consistently provides competitive or superior results, achieving the best performance on all four tasks. Considering both stability and simplicity, applying a single quantile remains the most effective and robust strategy.

\subsection{Ablation on automatic $\bar{\sigma}$}\label{app:automatic_sigma}
Since \textbf{SCLoRA} uses a fixed quantile $q$, the corresponding spectral bound $\bar{\sigma}$ is also fixed. Instead of using this fixed value, one may alternatively treat $\bar{\sigma}$ as a learnable parameter so that it can be automatically learned during training. We conduct an ablation study where $\bar{\sigma}$ is made learnable, and the results are reported in \Cref{tab:abla_learnable_sigma}.
The results show that using a learnable $\bar{\sigma}$ does not lead to a substantial degradation in performance. However, the fixed $\bar{\sigma}$ generally yields more stable and consistently better results across tasks. In addition, since a learnable $\bar{\sigma}$ introduces extra parameters and potential computational overhead, we choose to use a fixed constant value rather than adopting the automatically learned variant.

\begin{table}[t]
\centering
\small
\begin{tabular}{lcccc}
\toprule
Method & CoLA & RTE & MRPC & STS-B \\
\midrule
Learnable $\bar{\sigma}$ & 64.21 & 82.43 & 90.28 & 91.10 \\
\rowcolor{blue!10}\textbf{SCLoRA} & \textbf{64.79} & \textbf{83.15} & \textbf{90.32} & \textbf{91.22} \\
\bottomrule
\end{tabular}
\caption{Ablation study on automatic $\bar{\sigma}$}
\label{tab:abla_learnable_sigma}
\end{table}

\section{Orthogonal Regularization on Parameterized Singular Vectors}\label{app:orthogonality}
\begin{figure}[t]
    \centering
    \subfigure[Orthogonal loss of $U$]{\includegraphics[width=0.49\columnwidth]{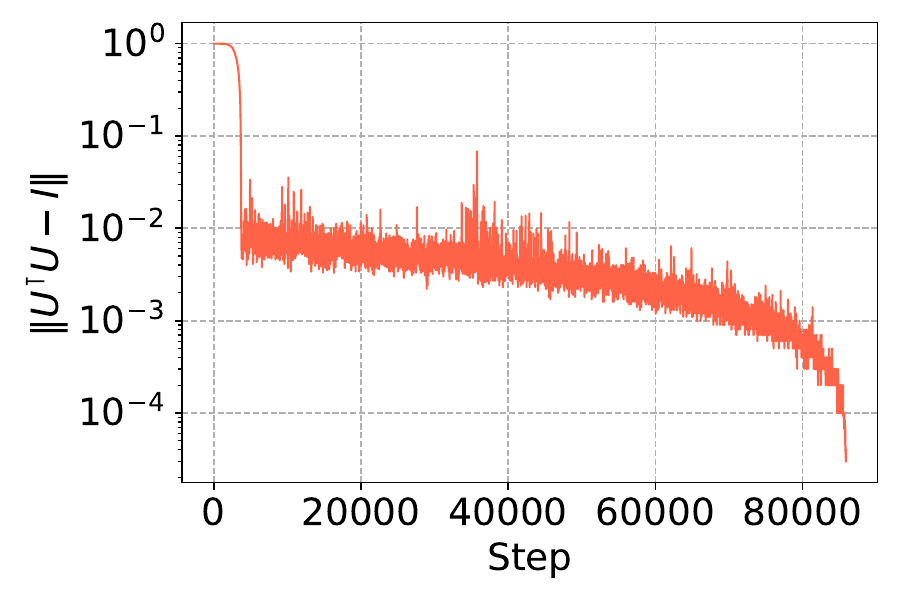}} 
    \subfigure[Orthogonal loss of $V^\top$]{\includegraphics[width=0.49\columnwidth]{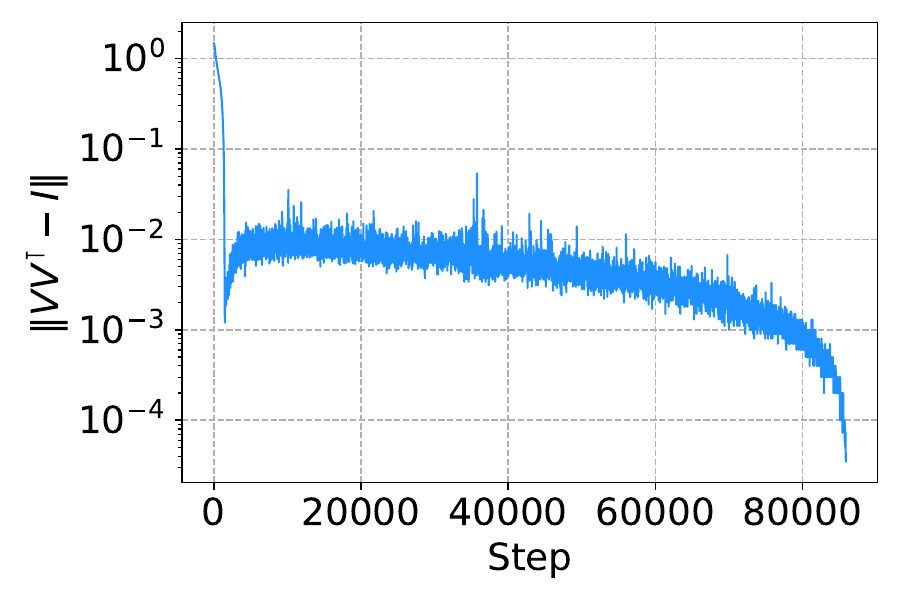}}
    \caption{The orthogonal loss curves of parameterized singular vectors $U$ and $V$ when fine-tuning RoBERTa\textsubscript{base} on STS-B dataset}
    \label{fig:ortho_loss}
\end{figure}

As shown in~\Cref{fig:sens_quantile}, when the $q$ and corresponding $\bar{\sigma}$ exceeds a certain level, the fine-tuning performance remains stable, and the performance on the pre-training task is largely recovered. In addition, we further treat the $\bar{\sigma}$ as a learnable variable and perform auto-tuning.
As presented in~\Cref{tab:abla_learnable_sigma}, making the quantile learnable maintains stable overall performance, and no performance degradation is observed. However, under our experimental setup, the learnable $\bar{\sigma}$ setting does not yield clear additional gains over the fixed setting. For simplicity and consistency, we therefore adopt the fixed $\bar{\sigma}$ configuration as the default choice.

\begin{table*}[t]
\centering
\small
\setlength{\tabcolsep}{4.7pt}
\begin{tabular}{l | cccccccc} \toprule
Method & MNLI  & SST-2  & CoLA  & QQP  & QNLI  & RTE  & MRPC  & STS-B  \\ \midrule
LoRA & 105.9/24.9 & 18.2/24.9 & 2.3/24.9 & 98.1/24.9 & 28.3/24.9 & 0.7/24.9 & 1.0/12.5 & 1.6/24.9 \\
PiSSA & 106.2/24.9 & 18.1/24.9 & 2.3/24.9 & 98.1/24.9 & 28.2/24.9 & 0.7/24.9 & 1.0/12.5 & 1.5/24.9 \\
AdaLoRA & 123.4/25.6 & 21.1/25.6 & 2.7/25.6 & 114.4/25.6 & 33.1/25.6 & 0.8/25.6 & 1.3/13.1 & 1.8/25.6 \\
MiLoRA & 106.0/24.9 & 18.1/24.9 & 2.3/24.9 & 98.1/24.9 & 28.2/24.9 & 0.7/24.9 & 1.0/12.5 & 1.5/24.9 \\
\textbf{SCLoRA} & 128.9/25.2 & 22.1/25.2 & 2.8/25.2 & 119.4/25.2 & 34.3/25.2 & 0.8/25.2 & 1.3/12.8 & 1.9/25.2 \\ \bottomrule
\end{tabular}

\caption{Comparison of training time (min per epoch) and peak GPU usage (GB) with RoBERTa\textsubscript{base}}
\label{tab:runtime_gpu}
\end{table*}

\begin{table*}[t]
\centering
\small
\setlength{\tabcolsep}{4.7pt}
\begin{tabular}{l | cccccccc} \toprule
Method & MNLI  & SST-2  & CoLA  & QQP  & QNLI  & RTE  & MRPC  & STS-B  \\ \midrule
LoRA & 0.85/0.3 & 0.08/0.3 & 0.10/0.3 & 3.46/0.3 & 0.47/0.3 & 0.03/0.3 & 0.05/0.3 & 0.14/0.3 \\
PiSSA & 0.84/0.3 & 0.08/0.3 & 0.09/0.3 & 3.47/0.3 & 0.47/0.3 & 0.03/0.3 & 0.04/0.3 & 0.13/0.3\\
AdaLoRA & 1.32/0.3 & 0.13/0.3 & 0.15/0.3 & 5.43/0.3 & 0.74/0.3 & 0.05/0.3 & 0.07/0.3 & 0.21/0.3 \\
MiLoRA & 0.84/0.3 & 0.08/0.3 & 0.09/0.3 & 3.47/0.3 & 0.47/0.3 & 0.03/0.3 & 0.04/0.3 & 0.13/0.3  \\
\textbf{SCLoRA} & 1.01/0.3 & 0.10/0.3 & 0.12/0.3 & 4.15/0.3 & 0.57/0.3 & 0.04/0.3 & 0.05/0.3 & 0.16/0.3 \\
\bottomrule
\end{tabular}
\caption{Comparison of inference time (min per epoch) and peak GPU usage (GB) with RoBERTa\textsubscript{base}}
\label{tab:inferencetime_gpu}

\end{table*}

\begin{table}[t]
\centering
\small
\setlength{\tabcolsep}{4.7pt}
\begin{tabular}{l | cc} \toprule
Method & LLaMA-7B & 	LLaMA2-7B \\
\midrule
LoRA & 	74.7/0.53 & 77.6/0.54 \\
\textbf{SCLoRA}	 & 79.4/0.51  & 81.0/0.52 \\
\bottomrule
\end{tabular}
\caption{Comparison of training time (min per epoch) and peak GPU usage (GB) with LLaMA family}
\label{tab:runtime_gpu_llama}
\end{table}

\Cref{fig:ortho_loss} shows the orthogonal loss curve of parameter singular vectors $U$ and $V$ of RoBERTa\textsubscript{base} fine-tuned on STS-B dataset. The singular vectors are orthogonally optimized as indicated by the consistent reduction in orthogonal loss.

\section{Empirical Complexity Analysis}\label{app:empirical_complexity}
% \section{Comparison of Computational Complexity}\label{app:complexity}
% \paragraph{Increased number of trainable parameters.}
% For the hidden dimension $d$ and the rank of adapter $r$, conventional LoRA learns $2dr$ parameters per layer. \textbf{SCLoRA} introduces only an additional $r$ parameter to the existing LoRA, resulting in a total of $2dr+r$ parameters per layer. For instance, when $r=2$ as shown in~\Cref{tab:parmeter_deberta_r2}, conventional LoRA learns 331,776 parameters, whereas \textbf{SCLoRA} learns a total of 331,920 parameters—the additional parameters are only 144.
% \paragraph{Increased computational cost.}
% In training stage, conventional LoRA has a time complexity of $\mathcal{O}(d^2r)$, \textbf{SCLoRA} has $\mathcal{O}(d^2r+dr+r^3)$, which is similar to existing parameterized SVD-based LoRA. However, since $r \ll d$ in LoRA (For instance, when $r = 32$, it corresponds to only about $0.78\%$ of the full dimension ($d = 4096$) in LLaMA-2.), the additional training cost is negligible, and we further report the empirical training time in the following paragraph.
% In inference stage, conventional LoRA involves computations with complexity $\mathcal{O}(d^2r)$ per layer, whereas \textbf{SCLoRA} includes the process of multiplying the parameterized singular value with the parameterized singular vector, resulting in an additional complexity of $\mathcal{O}(d^2r+dr)$, which is comparable. Note that the explicit SVD is performed only once during the initialization stage and is not required during fine-tuning.
% \paragraph{Empirical complexity analysis.}
We conduct additional experiments on empirical time complexity and GPU usage during fine-tuning and inference phase. \Cref{tab:runtime_gpu} summarizes the empirical training time (min per epoch) and peak GPU usage (GB) of RoBERTa\textsubscript{base} fine-tuned on GLUE tasks. The GPU usage showed a very slight increase compared to the original LoRA, and the additional runtime occurs in \textbf{SCLoRA} and AdaLoRA. This increase arises from the orthogonal regularization of singular vectors generated by parameterized SVD. However, fine-tuning typically requires fewer epochs, and considering the improved performance and the ability to retain pre-trained knowledge compared to the baseline model, this increase is negligible. \Cref{tab:inferencetime_gpu} reports the  empirical inference time (min per epoch) and peak GPU usage (GB) of RoBERTa\textsubscript{base} fine-tuned on GLUE tasks. \textbf{SCLoRA} exhibits a marginal increase in inference latency and peak GPU memory relative to vanilla LoRA; nonetheless, both metrics remain lower than those of AdaLoRA and are practically comparable.
In large models such as LLaMA, the primary computational complexity stems from the high dimensionality of the base model. For instance, we empirically measure the computational complexity in LLaMA-7B and LLaMA2-7B, and we report ``average accuracy on commonsense reasoning / the steps per second during fine-tuning'' in~\Cref{tab:runtime_gpu_llama}. Since $r \ll d$, the training process runs at approximately 96\% of the original speed, which is not a significant difference.

\end{document}